%% file: ArXiv/main_arxiv.tex
\documentclass[11pt, letterpaper]{article}
\usepackage{ArXiv/arxiv}

\title{\LARGE Towards Theoretical Understanding of Transformer Test-Time Computing: Investigation on In-Context Linear Regression}

\author{
Xingwu Chen\thanks{Equal contribution.} \thanks{School of Computing \& Data Science, The University of Hong Kong. Email: \email{xingwu}{connect.hku.hk}}
\qquad Miao Lu\footnotemark[1] \thanks{Department of Management Science and Engineering, Stanford University. Email: \email{miaolu}{stanford.edu}}
\qquad Beining Wu \thanks{Department of Statistics, University of Chicago. Email: \email{beiningw}{uchicago.edu}}
\qquad Difan Zou \thanks{School of Computing \& Data Science, Institute of Data Science, The University of Hong Kong. Email: \email{dzou}{cs.hku.hk} }
}
\date{\small{\today}}

\begin{document}


\maketitle

\begin{abstract}
    Using more test-time computation during language model inference, such as generating more intermediate thoughts or sampling multiple candidate answers, has proven effective in significantly improving model performance. This paper takes an initial step toward bridging the gap between practical language model inference and theoretical transformer analysis by incorporating randomness and sampling. We focus on in-context linear regression with continuous/binary coefficients, where our framework simulates language model decoding through noise injection and binary coefficient sampling. Through this framework, we provide detailed analyses of widely adopted inference techniques. Supported by empirical results, our theoretical framework and analysis demonstrate the potential for offering new insights into understanding inference behaviors in real-world language models.
\end{abstract}

\tableofcontents

\newpage








\input{icml2025/sections/intro}

\input{icml2025/appendix/related_work_continued}
\input{icml2025/sections/problem_setting}

\input{icml2025/sections/framework}
\input{icml2025/sections/analysis_continuous}
\input{icml2025/sections/analysis_binary}

\input{icml2025/sections/experiments}

\input{icml2025/sections/conclusions}

\bibliography{ArXiv/reference}

\bibliographystyle{ims}

\newpage

\appendix 

\input{icml2025/appendix/experiments}

\input{icml2025/appendix/transformer}
\input{icml2025/sections/continuous_continued}

\input{icml2025/appendix/continuous}

\input{icml2025/appendix/binary}
\input{icml2025/appendix/prompt}

\end{document}

%% file: icml2025/sections/intro.tex
\section{Introduction}

Transformer-based \citep{vaswani2017attention} large language models (LLMs) have demonstrated impressive general-purpose capabilities, representing state-of-the-art architectures in natural language processing \citep{dubey2024llama,guo2025deepseek,achiam2023gpt} and increasingly in other domains such as computer vision \citep{peebles2023scalable,agarwal2025cosmos}.
While scaling laws for LLM training \cite{kaplan2020scaling} have described their performance with respect to the train-time compute (i.e. model size, data size, and training time, e.g.), leveraging additional test-time computation of the pretrained LLMs, such as extend reasoning length by generating additional intermediate thoughts \citep{wei2022chain,guo2025deepseek,o1}  or sampling multiple candidate answers and aggregating to obtain the best one \citep{cobbe2021training, wangSelfConsistencyImprovesChain2023}, has recently demonstrated great potential for further enhancing their reasoning capabilities. However, despite the success of scaling up test-time computing for LLMs, the theoretical understanding of transformer models, even for the relatively simpler linear cases, for such successes remains quite limited.


Due to the success of LLMs itself, a huge body of recent theory works has emerged, aiming at understanding the hidden mechanisms of transformers from other angles. These works have been focused on seeking to explain the model's capabilities in memorization \citep{mahdaviMemorizationCapacityMultiHead2023,kim2023provable}, in-context learning (ICL) \citep{vonoswaldTransformersLearnIncontext2023,zhang2023trained,huang2025transformers}, function approximation power \citep{takakuraApproximationEstimationAbility2023,malachAutoRegressiveNextTokenPredictors2023}, algorithm simulation  \citep{chen2024can,fu2023transformers,liu2024learn}, and the training dynamics \citep{yangInContextLearningRepresentations2024,zhang2023trained,chen2024training} for  transformers initialized from scratch, to name a few. 
Most of these works consider simplified settings with linear attention \citep{vonoswaldTransformersLearnIncontext2023} and focus on how transformers can \textit{directly} leverage their output activations to solve specific tasks like in-context linear regression \citep{gargWhatCanTransformers2023},  ignoring the sampling and tokenization procedure for LM decoding, creating substantial gaps between theoretical analysis and practical LLM applications.


One of the main gap between prior theoretical works and LLM used in practice is that, prior theoretical works typically focus on transformers with deterministic decoding procedures, where the model output is fixed for a given prompt. 
In practice, many  inference techniques for scaling up test-time computing, such as majority voting \citep{wangSelfConsistencyImprovesChain2023}, best-of-N sampling (BoN) \citep{cobbe2021training}, and tree of thoughts (ToT) \citep{yao2024tree}, rely on probabilistic sampling procedures in real-world LLMs: given a prompt, the model predicts subsequent tokens by first computing a distribution over potential candidates and then sampling from it. This gap between the theoretical setups and the real-world LLM behavior hinders us towards understanding and analyzing of the success of transformer test-time computation.

\textbf{Our contributions.} In this work, we aim to bridge the gap between practical language model (probabilistic) inference and theoretical transformer analysis, providing initial theoretical insights into transformer test-time computation. Specifically, we examine the in-context linear regression task with continuous/binary coefficients,
simulate LLMs' sampling decoding procedure by injecting random noise (continuous case) or conducting discrete sampling (binary case) based on the model's original output, using the processed tokens for subsequent sampling decoding steps. We then conduct analysis towards test-time computation of transformers based on our theoretical framework. The main contributions of this paper are highlighted as follows:


\begin{itemize}[leftmargin=*]

\item We take an initial step toward bridging the gap between practical language model inference and theoretical transformer analysis by incorporating randomness and sampling. Our framework simulates language model decoding through noise injection and binary coefficient sampling, exhibiting trends similar to real-world LLMs' inference, as demonstrated in Fig \ref{fig:gsm8k_sparse_comp}.

\item Through our framework, we conduct detailed analysis of how test-time computation plays a role in our reasoning framework, including reasoning steps and sampling number, which can be applied to widely adopted inference techniques such as majority voting, ensembling, and chain-of-thought prompting. 

\item We validate our theoretical analysis through extensive experiments. Furthermore, we attempt to predict real-world LLM performance using our theoretical framework. The results demonstrate the potential of applying our theoretical framework for practical LLM behavior analysis.


\end{itemize}
\begin{figure}
\vskip -0.0in
\begin{center}
\centerline{\includegraphics[width=\columnwidth]{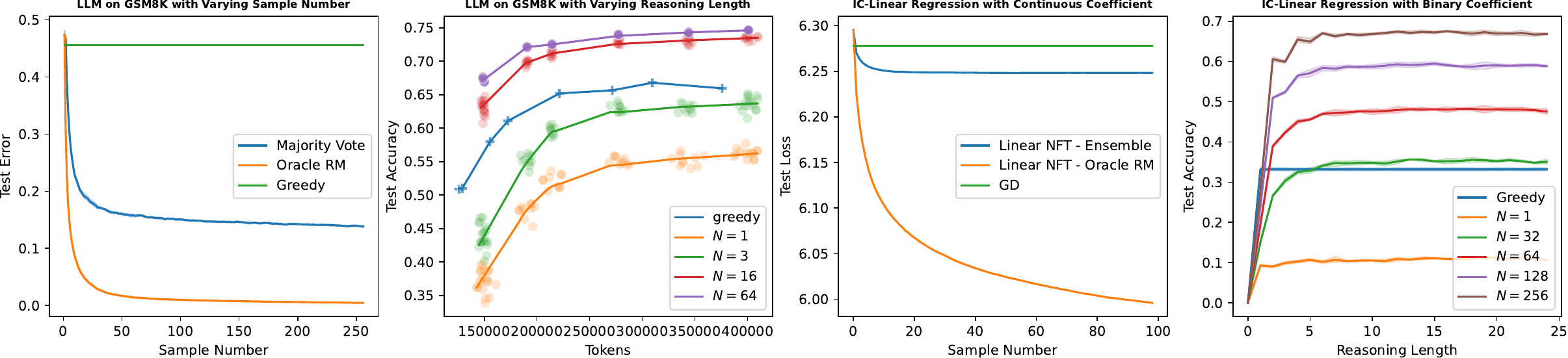}}
\vskip -0.1in
\caption{Comparison between real-world LLM's inference (above) and our designed sampling framework (below) for different sample numbers $N$ and reasoning lengths. Our framework simulates language model decoding through noise injection and binary coefficient sampling, exhibiting trends similar to real-world LLMs' inference, Details can be found in \autoref{appdix:exp_details} }
\label{fig:gsm8k_sparse_comp}
\end{center}
\vskip -0.4in
\end{figure}

%% file: icml2025/appendix/related_work_continued.tex
\section{Related Works}\label{sec:_related_work_continued}

\textbf{Scaling test-time computing in LLMs.}
Scaling test-time computing has demonstrated tremendous empirical success in LLMs, especially for reasoning tasks \citep{o1}.
Recent research on increasing test-time computing in LLMs primarily focuses on the following two aspects \citep{snell2024scaling}: (i) generating longer reasoning paths, including chain-of-thought (CoT) prompting that elicits intermediate reasoning steps \citep{wei2022chain,kojima2022large} and self-refinement methods that iterate on previously generated content \citep{madaan2303self, saunders2022self,kumar2024training}; and (ii) generating multiple potential reasoning paths and selecting the optimal one through the methods such as consistency-based selection \citep{wangSelfConsistencyImprovesChain2023}, reward-guided choosing \citep{stiennon2020learning, liu2020learning, cobbe2021training,dong2023raft}, reasoning tree search \citep{yao2024tree,zhou2023language}, etc.
Empirical studies demonstrate that increased test-time computation consistently improves model performance \cite{snellScalingLLMTestTime2024,yue2024inference,o1}, suggesting the existence of inference scaling laws \citep{wuInferenceScalingLaws2024}.
Nevertheless, the theoretical analysis of inference-time computing and its scaling law remains quite open.

\textbf{Theory for transformer test-time computing.}
Inspired by the empirical success of the inference-time computing techniques of LLMs, recently there have been a few works trying to demystify the mechanism behind it through analysis on theoretical tasks and simple transformer models.
Both \cite{wen2024sparse, kim2024transformers} consider how to train a one-layer transformer that utilizes CoT reasoning to efficiently solve the $k$-parity learning task, which provably improves over the same one without using CoT reasoning. 
\cite{hu2024unveiling} studies the statistical properties of CoT prompting and its variants including majority vote and tree-of-thought (ToT). 
However, their analysis is model agnostic and does not consider concrete transformer models compared with our work. 
The mostly related to our paper is the work of \cite{huang2025transformers} who considers a one-layer transformer to solve in-context linear regression task with continuous coefficient.
They show that the transformer can be well trained to perform vanilla multi-step GD with CoT. 
However, the fundamental difference between the study of \cite{huang2025transformers, wen2024sparse, kim2024transformers} and ours is that we propose to include randomness in the inference stage of the transformer models, which then allows us to go further and study more sophisticated test-time computing methods that involve randomly sampling multiple reasoning or CoT paths.


\textbf{Theory for in-context learning by transformers.}
In-context learning (ICL) \citep{brown2020language} is a key capability of LLMs which means that the model is able to answer a new query provided with a few query-answer demonstrations of the similar tasks without updating the model parameters.
The empirical success of ICL methods has sparked a long line of theoretical research for the ICL ability of transformers. 
Most of these theoretical research builds on the in-context learning framework of  \cite{gargWhatCanTransformers2023}, where input-output pairs are formalized as $\{(\xb_i,f(\xb_i))\}_{i = 1}^{n}$, and the model (typically, transformers) is required to learn the unknown function $f(\cdot)$ from the context without updating the parameters. This framework enables theoretical analysis of transformers across multiple dimensions: expressive power \citep{bai2024transformers,guoHowTransformersLearn2023}, mechanistic understanding \citep{giannouHowWellCan2024, vonoswaldTransformersLearnIncontext2023, ahn2023transformers,li2025robustness}, and training dynamics \citep{zhang2023trained, huangContextConvergenceTransformers2023, chen2024training,wu2023many, zhang2025transformer}. While most existing research treats transformer decoding as a deterministic process, theoretical understanding of test-time computation for transformer ICL remains in its infancy.

%% file: icml2025/sections/problem_setting.tex
\subsection{Preliminaries and More Backgrounds}
This section outlines the problem setups. We first detail transformers' inference mechanism, emphasizing \emph{sampling-based} techniques for enhancing test-time computation. We then introduce in-context linear regression, the theoretical task central to our study.
\subsubsection{Transformer and Sampling-based Test-time Computing}\label{subsec:_transformer_test_time_compute}
A transformer \citep{vaswani2017attention} is an auto-regressive sequence-to-sequence model that predicts the next token's distribution, i.e., $p(x_{t+1}|x_t,\cdots,x_1)$. 
It maps the representation of the last token  $x_{t+1}$ to a softmax distribution over the vocabulary space $\cV$ to determine the probability of $x_{t+1}$.

The above inference mechanism can be abstracted in the following way. 
Given the current input sequence embedding $\Hb_t = (\hb_1,\cdots,\hb_t)\in\mathbb{R}^{d_{\mathrm{e}}\times t}$, one \emph{iteratively} performs the following two steps: 
\begin{itemize}[leftmargin=*,nosep]
    \item Compute and extract the hidden state for the last position $t$, i.e., $\widetilde{\hb}_t = \mathtt{TF}_{\theta}(\Hb_t)$, where $\mathtt{TF}_{\theta}$ denotes the stacked transformer blocks in the whole architecture.
    \item Sample the next token $x_{t+1}$ (and thus the embedding of the next token $\hb_{t+1}$) based on a probability distribution returned by a sampling algorithm inputted with $\widetilde{\hb}_t$, i.e., $\hb_{t+1} \leftarrow \mathtt{Sampling\_Alg}(\widetilde{\hb}_t)$.
\end{itemize}

\textbf{Sampling-based test-time computation.}
As previously introduced, the probabilistic nature of the computation procedure can introduce randomness into the inference process, which is key to an array of techniques for scaling up test-time computing in order to boost the performance of large language models for various tasks, including Best-of-N sampling (BoN) \citep{stiennon2020learning, nakano2021webgpt, dong2023raft}, majority vote \citep{wang2022self}, etc.
Notably, these methods typically sample $N$ independent reasoning trajectories through the above decoding mechanism and choose the one with the highest value of a given reward model or the most consistent one across all candidates.




\subsubsection{Theoretical Task: In-context Linear Regression.}\label{subsec:_theoretical_task}


We explore how sampling-based test-time computing can enhance transformer performance by focusing on \emph{in-context linear regression}, a common problem setup \citep{akyurek2022learning, von2023transformers, zhang2023trained, chen2024transformers}. In-context learning (ICL \cite{brown2020language}) involves auto-regressive models inferring answers from few task demonstrations. we consider the following general setup: first drawing the ground truth parameter from the prior $\wb^{\ast}\sim p_{\wb}(\cdot)$, then
\begin{align}
    (\xb_i, y_i)\sim \mathbb{D}_{\wb^{\ast}}, y_i = \xb_i^\top \wb^{\ast} + \epsilon_i,\  \forall i\in[n] ,\label{eq:_in_context_linear_regression}
\end{align}
where $p_{\wb}$ denotes the prior distribution of the regression tasks, $\epsilon_i$ is the i.i.d. random noise, and $n\in\NN$ is the size of the in-context dataset.
The goal of in-context linear regression is to use transformers to make predictions regarding the true label $\xb_{\mathrm{query}}^\top \wb^{\ast}$ associated with another covariate $\xb_{\mathrm{query}}\sim \cN(0,\Ib_d)$ when prompted with the in-context dataset $(\xb_1,y_1,\cdots,\xb_n,y_n)$ concatenated with the query $\xb_{\mathrm{query}}$.
Towards such a goal, this work aims to establish a theoretical framework that allows one to principally investigate how sampling-based techniques for scaling up test-time computing could benefit the predictions, thus boosting the performance of solving the task.

%% file: icml2025/sections/framework.tex
\section{Scaling Test-time Computation for In-Context Regression}\label{sec:_test_time_computing_framework}

In this section, we introduce our theoretical framework for studying sampling-based test-time computing of transformers (Section~\ref{subsec:_transformer_test_time_compute}) through in-context linear regression (Section~\ref{subsec:_theoretical_task}).
We present our framework in Section~\ref{subsec:_theoretical_framework}. 
After that, we study two instances of the in-context linear regression task \eqref{eq:_in_context_linear_regression}, depending on the types of the task prior $p_{\wb}$, to design concrete sampling algorithms for inference.  


\subsection{A Theoretical Framework}\label{subsec:_theoretical_framework}

We begin by noticing that most of the existing prior works on in-context linear regression by transformers are \emph{incapable} for studying sampling-based test-time computing due to the lack of (i) randomness of the output of the transformer architecture they study; (ii)  chain-of-thought (CoT) style multi-step reasoning in the outputs. 
To handle the challenge, we explicitly construct an inference mechanism that involves both randomness and auto-regressive CoT reasoning to solve in-context linear regression tasks.
Specifically, motivated by the recent work of \cite{huang2025transformers}, we consider the specific goal of \emph{in-context coefficient prediction}, where the final output of the transformer reasoning path is a prediction $\widehat{\wb}$ of the task coefficient $\wb^{\ast}$.
The transformer inference mechanism is designed to output stochastic  reasoning paths, and different sampling-based test-time computing techniques correspond to how to aggregate different reasoning paths.

\textbf{Inputs and transformer architecture.}
Given the in-context dataset $(\xb_1,y_1,\cdots,\xb_n,y_n)$, the prompt to the transformer (defined later) is the following matrix in $\RR^{d_e\times (n+1)}$, 
\begin{align}
    \!\!\!\!\!\!\!\!\Hb_0 = \left(\begin{matrix}
        \xb_1&\cdots & \xb_n &\boldsymbol{0} \\
        y_1&\cdots&y_n& 0 \\
        \boldsymbol{0}&\cdots&\boldsymbol{0}&\wb_0\\
        0&\cdots&0&1
    \end{matrix}\right) := \left(\begin{matrix}
        \Xb^\top & \boldsymbol{0}\\
        \yb^\top & 0 \\
        \boldsymbol{0} & \wb_0 \\
        \boldsymbol{0} & 1
    \end{matrix}\right),\label{eq:_input}
\end{align}
where the dimension of the embedding $d_e = 2d+2$. 
We denote $\Xb^\top = (\xb_1,\cdots,\xb_n)\in\RR^{d\times n}$ as the collection of covariates, and denote $\yb^\top = (y_1,\cdots,y_n)\in\RR^{1\times n}$  as the collection of labels.
We input an initial guess of the coefficient, denoted by $\wb_0$, and we $\wb_0=\boldsymbol{0}$ without loss of generality.
Note that such a prompt embedding format which separates the space of data and the space of weight predictions follows the convention of \cite{bai2024transformers, huang2025transformers} in order to facilitate theoretical analysis.

The model we consider is a one-layer self-attention module equipped with residual connection \citep{von2023transformers, zhang2023trained, ahn2023linear, huang2025transformers}:
\begin{align}
    \!\!\!\!\!\!\!\!\!\!\mathtt{TF}_{\theta}(\Hb):=\Hb + \Vb\Hb\cdot\frac{\Hb^\top \Wb \Hb}{n}:\RR^{d_e\times \ast}\mapsto\RR^{d_e\times \ast}.\label{eq:_transformer}
\end{align}
where $\theta = \{\Vb,\Wb\}$ denotes the parameters. 
Here $\Vb\in\RR^{d_e\times d_e}$ represents the consolidation of the projection and value matrices in a standard transformer block, and $\Wb\in\RR^{d_e\times d_e}$ denotes the consolidation of the key and query matrices. 

\textbf{Sampling-based auto-regressive inference mechanism.}
With the model \eqref{eq:_transformer} and the prompt \eqref{eq:_input}, we consider the following mechanism of inference that mimics a real LLM.
\begin{definition}[Inference mechanism]\label{def:_inference_mechanism}
Given a prompt embedding matrix $\Hb_0$, for each $\ell\in\NN$, we iteratively sample the embeddings for the next token as following:
\begin{itemize}[leftmargin=*]
    \item Compute $\widetilde{\Hb}_{\ell} = \mathtt{TF}_{\theta}(\Hb_{\ell})$ with $\mathtt{TF}_{\theta}(\Hb_{\ell})$ defined in \eqref{eq:_transformer};
    \item Extract $\widetilde{\hb}_{\ell}$ from $\widetilde{\Hb}_{\ell}$ last column, i.e., $\widetilde{\hb}_{\ell} = (\widetilde{\Hb}_{\ell})_{:,-1}$;
    \item Sample the embedding vector for the next token via $\mathtt{Sampling\_Alg}$, i.e., $\hb_{\ell+1} \leftarrow \mathtt{Sampling\_Alg}(\widetilde{\hb}_{\ell})$; 
    \item Concatenate to obtain the embedding matrix for the new sequence of length $\ell+1$, i.e., $\Hb_{\ell+1} = (\Hb_{\ell}, \hb_{\ell+1})$.
\end{itemize}
\end{definition}
Here $\mathtt{Sampling\_Alg}(\cdot)$ is to be determined that assigns the distribution of the next token (embedding) conditioning on the last token's embedding output by the transformer.
Note that the output of the above mechanism is a joint result of the transformer model and the sampling algorithm.

Towards the goal of in-context weight prediction for \eqref{eq:_in_context_linear_regression}, we introduce the following proposition, which shows that the transformer architecture together with a proper sampling algorithm  can implement variants of \emph{noisy gradient descent}.

\begin{proposition}[Definition~\ref{def:_inference_mechanism} can implement noisy GD]\label{prop:_construction}
    There exists a transformer instance of \eqref{eq:_transformer} denoted by $\mathtt{TF}_{\theta_{\mathtt{GD}}}$ and a type of sampling algorithm $\mathtt{Sampling\_Alg}$ such that given prompt $\Hb_0$ defined in  \eqref{eq:_input}, the output embedding after $t$ iterative generations $\Hb_t$ according to Definition~\ref{def:_inference_mechanism} satisfies $(\Hb_t)_{:,n+\ell} = (\boldsymbol{0}^\top,0,\wb_{\ell}^\top, 1)^\top$ with
    \begin{align}
        \wb_{\ell} \sim p\left(\cdot\middle|\wb_{\ell-1} \!-\! \frac{\eta}{n}\cdot \Xb^\top\big(\Xb \wb_{\ell-1}\!-\!  \yb\big)\right),\!\! \forall 1\leq \ell\leq t,
    \end{align}
    where the conditional distribution $p(\cdot|\cdot )$ is specified by the sampling algorithm $\mathtt{Sampling\_Alg}$.
\end{proposition}

This proposition is mainly motivated by the recent work of \cite{huang2025transformers}.
Please refer to Appendix~\ref{subsec:_proof_prop_construction} for a detailed proof of Proposition~\ref{prop:_construction}.
Proposition~\ref{prop:_construction} shows that the above inference mechanism is able to explicitly implement gradient-based iterative algorithms to predict the regression coefficient $\wb^{\ast}$.
We define the prediction of the regression coefficient after $t$ reasoning steps of one reasoning path as $\wb_t:=(\Hb_t)_{d+2:2d+1, n+t}$.
One special case of Proposition~\ref{prop:_construction} is a transformer that explicitly performs standard multi-step GD \citep{huang2025transformers}, i.e., $p(\cdot|x) = \delta_x(\cdot)$.
Please see Appendix~\ref{subsec:_special_case} for the details.

Now to theoretically understand the effectiveness of more sophisticated sampling-based test-time computing techniques, e.g., Best-of-N and majority vote, we go beyond \eqref{eq:_w_gd} and consider sampling algorithms that does introduce randomness into the reasoning path.
We formalize these test-time computing methods we study in this paper as following.


\begin{definition}[Sampling-based test-time computing techniques]\label{def:_test_time_computing}
    Given a transformer $\mathtt{TF}_{\theta}$ and a sampling algorithm that jointly satisfy Proposition~\ref{prop:_construction}, together with a prompt embedding matrix $\Hb_0$ in \eqref{eq:_input}, a CoT reasoning length limit $t\in\NN_+$, and a sampling budget $N\in\NN_+$, we consider the following test-time computing methods:
    \begin{itemize}[leftmargin=*,nosep]
        \item Firstly generate $N$ random predictions of the regression coefficient as $\{\wb^{(j)}_t\}_{j=1}^N$ (see Proposition~\ref{prop:_construction});
        \item Then aggregate the $N$ random outcomes $\{\wb^{(j)}_t\}_{j=1}^N$ by using one of the following options:
        \begin{enumerate}[nosep,leftmargin=*]
            \item Ensemble: $\wb_{\mathtt{avg}}:=N^{-1}\cdot \sum_{j=1}^N\wb^{(j)}_t$;
            \item Best-of-N: $\wb_{\mathtt{BoN}}:=\argmax_{\{\wb^{(j)}_t\}_{j=1}^N}R(\wb^{(j)}_t)$ where $R(\cdot):\RR^d\mapsto\RR$ is certain reward function;
            \item Majority vote: 
            $\wb_{\mathtt{mv}}:=\argmax_{\{\wb^{(j)}_t\}_{j=1}^N}\!\!\mathtt{Occur}(\wb^{(j)}_t)$, where $\mathtt{Occur}(\cdot):\RR^d\mapsto\NN$ is a proper function that counts the occurrence of the input.
        \end{enumerate}
    \end{itemize}
\end{definition}

In the following Sections~\ref{subsec:_continuous} and \ref{subsec:_binary}, we instantiate the in-context linear regression task \eqref{eq:_in_context_linear_regression} to more concrete task priors, and investigate the effectiveness and the scaling law of the above test-time computing techniques. 
We also remark that in this paper we assume the existence of a transformer satisfying Proposition~\ref{prop:_construction} without explicitly training such one from scratch, which is left as an interesting future work.

\subsection{Case Study 1: In-context Linear Regression with Continuous Coefficient}\label{subsec:_continuous}

The first type of tasks we consider is the standard in-context linear regression with continuous regression coefficient sampled from a Gaussian distribution, i.e., $p_{\wb} = \cN(\boldsymbol{0}, \omega^2\cdot \Ib_d)$.
For this case, the specific type of sampling algorithms $\mathtt{Sampling\_Alg}$ we study is concluded in Algorithm~\ref{alg:_sampling_alg_continuous}.
\begin{algorithm}[H]
\caption{Sampling algorithm for in-context linear regression with continuous coefficient}\label{alg:_sampling_alg_continuous}
\begin{algorithmic}[1]
  \STATE \textbf{Input:} token embedding $\widetilde{\hb}$, noise level $\sigma\geq 0$, noise transformation function $\phi_\cdot(\cdot):\RR^d\times\RR^d\mapsto\RR^d$.
  \STATE Extract the coefficient $\widetilde{\wb}$ from $\widetilde{\hb}$, i.e., $\widetilde{\wb} = (\widetilde{\hb})_{d+2:2d+1}$
  \STATE Sample a noise vector $\boldsymbol{\xi} \sim \mathcal{N}(\boldsymbol{0}, \sigma^2 \cdot \Ib_d)$ \\
  \STATE Define $\wb \leftarrow \widetilde{\wb} + \phi_{\boldsymbol{\xi}}(\widetilde{\wb})$ \\
  \STATE \textbf{Output:} $\hb:=(\boldsymbol{0}, 0, \wb, 1)^\top$.
\end{algorithmic}
\end{algorithm}

Under sampling method Algorithm~\ref{alg:_sampling_alg_continuous}, Proposition~\ref{prop:_construction} is satisfied with $ x'\sim p(\cdot|x)$ given by $x'  = x + \phi_{\boldsymbol{\xi}}(x)$ for a Gaussian random seed $\boldsymbol{\xi}$ and some noise transformation function $\phi_{\boldsymbol{\xi}}$.
Recall that by Proposition~\ref{prop:_construction}, $\widetilde{\wb}$ output by the transformer is performing one-step gradient descent from the last prediction.
The intuition of studying Algorithm~\ref{alg:_sampling_alg_continuous} is that such a noisy version of the gradient descent could allow exploration of the loss landscape, and we aim to investigate whether  the test-time computing techniques in Definition~\ref{def:_test_time_computing} could properly aggregate the random gradient-based paths to achieve a better prediction than vanilla multi-step GD \eqref{eq:_w_gd} via less overfitting.
In this paper, we investigate the following two concrete and simple examples of the noise transformation function (NFT) $\phi_{\boldsymbol{\xi}}$.
Potential future works could investigate other types of $\phi_{\boldsymbol{\xi}}$.

\begin{example}[Constant NFT]\label{exp:_constant}
    $\phi_{\boldsymbol{\xi}}(\wb):=\boldsymbol{\xi}$, independent of the input $\wb$ and is homogeneous across reasoning steps.
\end{example}

\begin{example}[Linear NFT]\label{exp:_linear}
    $\phi_{\boldsymbol{\xi}}(\wb):= \boldsymbol{\xi}\boldsymbol{\xi}^\top\wb$, linear in the input predicted weight $\wb$ such that the sampling distribution has different shape based upon the current decoding result.
\end{example}


We consider the following test-time computing methods.

\textbf{Baseline: multi-step GD with CoT \eqref{eq:_w_gd}.} 
This is a transformer implementing a vanilla GD, without using Algorithm~\ref{alg:_sampling_alg_continuous} but directly using one-step GD as the next token. It is clear that this baseline is deterministic and does not require multiple samples.

\textbf{Ensemble.} We consider sample average of the predictions from $N$ reasoning paths. We denote the resulting prediction after $N$ sampling paths of  length $t$ as $\wb_{\mathtt{avg}}$.

\textbf{Best-of-N.} We also consider  BoN with the oracle reward model $R^{\star}(\wb) := -\|\wb - \wb^*\|_2^2$.
The resulting prediction accuracy gives an upper bound for other test-time computing method due to the usage of the truth.
We denote the resulting  prediction after $N$ sampling paths of  length $t$ by $\wb_{\mathtt{BoN}}$.


\subsection{Case Study 2: In-context Sparse Linear Regression in Discrete Space}\label{subsec:_binary}

Motivated by the practical setting where the candidate tokens lie in a discrete space, we also consider another case in which the coefficient is a sparse binary vector, denoted as $\wb^* \in \{0,1\}^d$ with $\|\wb^*\|_0 = k < d$. 
In this situation, we consider the following sampling algorithm $\mathtt{Sampling\_Alg}$, which performs sampling on a discrete space $\{0,1\}^d$ based on the predicted weight $\widetilde{\wb}$ in the transformer output. In algorithm~\ref{alg:_sampling_alg_binary}, the function $\mathtt{ClipNorm}(\cdot)$ first clips each element in $\widetilde{\wb}$ to be non-negative and then normalizes the resulting vector such that its elements sum to 1, i.e.,  $(\mathtt{ClipNorm}(\widetilde{\wb}))_i = \max\{\widetilde{w}_i,0\}/\sum_{i'=1}^d\max\{\widetilde{w}_{i'},0\}$. This resembles the softmax operation over a vocabulary set.
Then algorithm~\ref{alg:_sampling_alg_binary} simulates LLM decoding by sampling tokens based such a distribution. 
More specifically, given the distribution $p$, we sample the (embedded) next token $\wb$ as a $k$-sparse vector with non-zero coordinates sampled from $p$.
We treat the vector sparsity $k$ as a fixed parameter satisfying $1 \leq k < d$, with $k$ typically set to $1$ in practice. 
Such a discrete nature of these coefficients enables us to consider the method of majority vote among the sampling-based test-time computing strategies in Definition~\ref{def:_test_time_computing}. 
In this work, we compare majority vote to a baseline inference mechanism based on greedy decoding which does not utilize sampling.
\begin{algorithm}[H]
\caption{Sampling algorithm for in-context linear regression with binary coefficient}\label{alg:_sampling_alg_binary}
\begin{algorithmic}[1]
  \STATE \textbf{Input:} token embedding $\widetilde{\hb}$, coefficient sparsity $k\in[d]$.
  \STATE Initialize $\wb  \leftarrow  \mathbf{0}_d$ 
  \STATE Extract the coefficient $\widetilde{\wb}$ from $\widetilde{\hb}$, i.e., $\widetilde{\wb} = (\widetilde{\hb})_{d+2:2d+1}$
  \STATE Compute predicted distribution $p = \mathtt{ClipNorm}(\widetilde{\wb})$ 
  \STATE Sample $k$ different indices $(e_1,\dots,e_k)\subset [d]$ based on $p$ without replacement \label{line:_sample}
  \STATE Assign $w_{e_\ell} = 1$ for each $e_\ell\in\{e_1,\cdots,e_k\}$
  \STATE \textbf{Output:} $\hb:=(\boldsymbol{0}, 0, \wb, 1)^\top$.
\end{algorithmic}
\end{algorithm}

\textbf{Baseline: greedy decoding.}  
In the decoding step, instead of sampling $k$ items based on $p$ as depicted in Algorithm~\ref{alg:_sampling_alg_binary} (Line~\ref{line:_sample}), we opt to choose $k$ items with the highest $k$ probabilities under $p$ and set the corresponding indices of $\wb$ to 1. 
This mirrors the greedy decoding algorithm commonly used in practice.
We denote the resulting prediction after $t$ reasoning steps as $\wb^{\mathtt{greedy}}_t$.


\textbf{Majority vote.} Utilizing the discrete nature of the coefficients, we apply the $\mathtt{Occur}(\cdot)$ function to candidate answers, selecting the most frequent one as our majority vote (see Definition~\ref{def:_test_time_computing}). The prediction after sampling $N$ reasoning paths of length $t$ is denoted as $\wb{t, N}^{\mathtt{mv}}$.

Here we present theoretical results for Case Study 1 and 2 in Section~\ref{sec:continuous_analysis} and ~\ref{sec:binary_analysis} respectively, with 
numerical results in Section~\ref{subsec:_numerical}.

%% file: icml2025/sections/analysis_continuous.tex
\section{Analysis of In-context Linear Regression with Continuous Coefficient}\label{sec:continuous_analysis}

In this section, we establish the theoretical analysis for Section~\ref{subsec:_continuous}.
We measure the performance of any in-context coefficient prediction by its population risk under $\mathbb{D}_{\wb^{\ast}}$, i.e., $ L_{\mathbb{D}_{\wb^{\ast}}}(\wb):= (1/2)\cdot\mathbb{E}_{(\xb,y)\sim \mathbb{D}_{\wb^{\ast}}}[(y - \xb^\top\wb)^2]$,
which is equivalent to consider the following excess risk, 
\begin{align}
    \cE(\wb)&:=L_{\mathbb{D}_{\wb^{\ast}}}\!(\wb) \!-\! \!\!\inf_{\wb'\in\RR^d} \!\!L_{\mathbb{D}_{\wb^{\ast}}}\!(\wb)= \frac{1}{2}\!\cdot\!\|\wb - \wb^{\ast}\|_{\Hb}^2,
\end{align}
where $\Hb:= \mathbb{E}_{\xb\sim \mathbb{D}_{\wb^{\ast}}}\left[\xb\xb^\top\right]$ denotes the population covariance matrix. 
We denote the collection of label noise in the in-context data as $\boldsymbol{\epsilon} := \yb - \Xb\wb^{\ast}$. 
We also denote the eigenvalues of the population covariance matrix $\Hb$ as $\{\lambda_i\}_{1\leq i\leq d}$ in a non-increasing order.
Our analysis relies on standard assumptions on the data distribution \citep{bartlett2020benign}, which is presented in Assumption~\ref{ass:_data_distribution} due to space limit.
By the same reason, we present our results for a special case of $\Hb$ with polynomially decaying eigenvalues, and refer to the readers to the expressions of general $\Hb$ in Appendix~\ref{subsec:_continuous_continued}.


\textbf{Baseline: multi-step GD with CoT.}
The following result gives the excess risk bound for transformers implementing vanilla multi-step gradient descent \eqref{eq:_w_gd}.
This is a corollary of Theorem~\ref{thm:_gd_continuous} and is proved in Appendix~\ref{subsec:proof_cor_gd_continuous_polynomial_decay}.


\begin{proposition}\label{cor:_gd_continuous_polynomial_decay}
    Under the same assumptions and setups as in Theorem~\ref{thm:_gd_continuous}, by additionally assuming that the spectrum of $\Hb$ satisfies polynomially decaying, i.e., $\lambda_i = i^{-(r+1)}$ for some $r\geq 1$, then for any reasoning path length $t\lesssim \eta(r+1)^{(r+1)/2}d^{(r+1)/2}$, with probability at least $1-1/\mathrm{poly}(n)$,
    \begin{align}
        \mathbb{E}_{\boldsymbol{\epsilon},\wb^{\ast}}\left[\mathcal{E}(\wb_{\mathtt{GD}})\right]\lesssim \omega^2\cdot\left(\frac{1}{t\eta}\right)^{\frac{r}{r+1}} + \frac{\sigma_{\epsilon}^2}{n}\cdot (t\eta)^{\frac{1}{r+1}}.
    \end{align}
\end{proposition}


\paragraph{Aggregating by ensembling.} 
In this case, the final regression coefficient reasoned by the transformer test-time computing under the budget of CoT length $t$ and reasoning path number $N$ is explicitly given by $\wb_{\texttt{avg}}:=N^{-1}\cdot \sum_{j=1}^N \wb^{(j)}_t$, 
where each random reasoning path $\{\wb_{\ell}^{(j)}\}_{1\leq\ell \leq t}$ is i.i.d. generated according to Definition~\ref{def:_inference_mechanism} via a transformer satisfying Proposition~\ref{prop:_construction} and with Algorithm~\ref{alg:_sampling_alg_continuous}.
The following result gives the excess risk bound for this method with different choices of the NFT $\phi_{\bxi}$. 
The proof is in Appendix~\ref{subsec:_proof_cor_average_continuous}.

\begin{theorem}\label{cor:_average_continuous}
     Under the same assumptions and setups as in Theorem~\ref{thm:_average_continuous},  additionally assuming that the spectrum of $\Hb$ satisfies polynomially decaying, i.e., $\lambda_i = i^{-(r+1)}$ for some constant $r\geq 0$,  we have the following results.
     \begin{enumerate}[left=0mm]
         \item Constant noise transformation function (Example~\ref{exp:_constant}): taking the  reasoning length $t\lesssim \eta (r+1)^{(r+2)/2}n^{(r+1)/2}$, with probability at least $1-1/\mathrm{poly}(n)$, 
             \begin{align}
        \mathbb{E}\left[\mathcal{E}(\wb_{\mathtt{avg}})\right]\lesssim \omega^2\!\cdot\!\left(\frac{1}{t\eta}\right)^{\frac{r}{r+1}} \!+\! \frac{\sigma_{\epsilon}^2}{n}\cdot (t\eta)^{\frac{1}{r+1}} \!+\! \frac{\vartheta_{n,t}}{N}.
    \end{align}
         \item Linear noise transformation function (Example~\ref{exp:_linear}): taking the noise variance $\sigma^2\asymp d^{-1}$, the reasoning length $t>\sigma^{-2}\cdot \log 2$, with probability at least $1-1/\mathrm{poly}(n)$, 
        \begin{align}
            \mathbb{E}\!\!\left[\mathcal{E}(\wb_{\mathtt{avg}})\right]\!\lesssim\!\omega^2\!\cdot\! \widetilde{\lambda}^{\frac{r}{r+1}} \!+\!   \frac{\sigma_{\epsilon}^2}{n}\!\cdot\! \left(\frac{\eta(1-\sigma^2)}{\sigma^2}\right)^{\frac{1}{r+1}} \!\!\!\!+\! \frac{\varsigma_n}{N},
    \end{align}
    where $\widetilde{\lambda}:= \eta^{-1}(2t^{-1}+\sigma^2(1+2t^{-1})/(1-\sigma^2))$.
     \end{enumerate}
     Here the expectation is taken with respect to $\boldsymbol{\epsilon}$, $\wb^{\ast}$, and all the sampling noise $\bxi$ across different reasoning steps and paths.
     The explicit formula for the functions $\vartheta_{n,t}$ and $\varsigma_n$ are deferred to \eqref{eq:_vartheta} and \eqref{eq:_varsigma}, respectively.
\end{theorem}

    The above theorem reveals how the prediction accuracy evolves as the reasoning length $t$ and sample numbers $N$ increase. In particular, we make the following remarks
    (i) In the above excess risk, the terms $\vartheta_{n,t}/N$ and $\varsigma_{n}/N$ represent the error from sampling finitely many reasoning paths $N$. By taking $N$ large enough (see \eqref{eq:_n_c} and \eqref{eq:_n_l} in Corollary~\ref{cor:_average_continuous_restate}), the leading term of the excess risk would be the first two terms.
    (ii) By the result for Example~\ref{exp:_constant},  Algorithm~\ref{alg:_sampling_alg_continuous} with constant noise does not provide benefit compared with TF implementing vanilla GD (see Proposition~\ref{cor:_gd_continuous_polynomial_decay}).
    (iii) In contrast, we next show that with linear NFT Algorithm~\ref{alg:_sampling_alg_continuous} can prevent overfitting to noisy labels. Considering the following regime of the parameters,
    \begin{align}
        \!\!\!\!\omega,\sigma_{\epsilon}\asymp 1,\quad n\asymp \eta d,\quad \sigma^{2}\asymp d^{-1},\quad t \asymp \widetilde{t}\cdot\sigma^{-2},\label{eq:_regime}
    \end{align}
    risk bounds for the vanilla multi-step GD and the ensemble method (using linear NFT (Example~\ref{exp:_linear})) are as following, 
    \begin{align}
        &\mathbb{E}_{\boldsymbol{\epsilon},\wb^{\ast}}\left[\mathcal{E}(\wb_{\mathtt{GD}})\right] \lesssim \widetilde{t}^{\frac{1}{r+1}} \cdot \left(\eta d\right)^{-\frac{r}{r+1}},\\
        &\mathbb{E}_{\boldsymbol{\epsilon},\wb^{\ast},\bxi}\left[\mathcal{E}(\wb_{\mathtt{avg}})\right] \lesssim \left(\eta d\right)^{-\frac{r}{r+1}},\,\,\mathrm{if}\,\,N\geq \eta^{\frac{r}{r+1}}d^{\frac{2r+1}{r+1}}.
    \end{align}
    Notice that by the conditions in Proposition~\ref{cor:_gd_continuous_polynomial_decay} and Theorem~\ref{cor:_average_continuous}, all the above conclusions hold when $t = \widetilde{t}\cdot\sigma^{-2}$ is not exceeding the order of $\eta(r+1)^{(r+1)/2}n^{(r+1)/2}$, which, under the parameter regime \eqref{eq:_regime}, translates to $\widetilde{t} \lesssim d^{(r-1)/2}$.
    Thus we are able to observe that in the high-dimensional regime, vanilla GD method has the disadvantage of harmful overfitting to the label noise when the effective reasoning path length $\widetilde{t}$ is increasing, while the sampling-based test-time computing does not (see details in  Remark~\ref{rmk:_regime}).


%% file: icml2025/sections/analysis_binary.tex
\section{Analysis of In-context Sparse Linear Regression in Discrete Space}
\label{sec:binary_analysis}

In this section, we conduct a theoretical analysis for binary sparse in-context linear regression (Section~\ref{subsec:_binary}). Our strategy of studying and comparing the test-time computing methods is to analyze the probability of perfectly recovering the true coefficient, i.e., $\mathbb{P}(\wb_t^{\mathtt{greedy}} = \wb^{\ast})$ and $\mathbb{P}(\wb_{t, N}^{\mathtt{mv}} = \wb^{\ast})$. We use the notation $p(\wb_t = \wb) := \PP(\wb_t = \wb \mid \wb_{0}, \cD)$ to indicate the probability of weight $\wb$ after $t$ reasoning steps, conditioning on the initial state $\wb_{0}$ and the in-context dataset $\cD$ in a single reasoning path. We define $\cW = \{\wb \mid \wb \in \{0,1\}^d, \|\wb\|_0 = k\}$ and assume $\bx \sim \cN(0,\Ib_d)$ and label noise $\epsilon_i \sim \cN(0,\sigma{\epsilon}^2)$ with $\sigma_{\epsilon} > 0$. 

Our first result shows that if in a single reasoning path the prediction $\wb_t$ has a probability of recovering the truth higher than that of recovering any other coefficient, then majority vote recovers the truth with a probability converging to $1$ exponentially fast.
The proof is in Appendix~\ref{subec:_proof_th_sc_maj_vote}.

\begin{proposition}[Sample complexity for majority vote] 
\label{th:sc_maj_vote}
    Consider the binary sparse in-context linear regression task (Section~\ref{subsec:_binary}) and using majority vote with reasoning length $T$ and sampling number $N$. 
    The final prediction $\wb_{t, N}^{\mathtt{mv}}$ can asymptotically recover the truth $\wb^*$ with probability $1$ given sufficient sample size $N$ if for a single reasoning path 
    \begin{align}
        \!\!\!\!\!\!\!\!\!\Delta_t:=p(\wb_t = \wb^*) - \max_{\wb' \in \cW\setminus\{\wb^{\ast}\}} p(\wb_t = \wb')>0.\label{eq:_recovery_condition}
    \end{align}
    Under condition \eqref{eq:_recovery_condition}, it holds that
\begin{align} 
\PP\left(\wb_{t, N}^{\mathtt{mv}} = \wb^*\mid \wb_0, \cD\right) \geq  1 - |\cW|\cdot \exp\left(-N\Delta_t^2/2\right).
\end{align}
\end{proposition}
We remark that similar results of Proposition~\ref{th:sc_maj_vote} have also been proposed in \cite{wuInferenceScalingLaws2024}.
Here, we further provide more detailed analysis for the majority vote in our binary sparse linear regression task, show its dependence on the in-context example number $n$, reasoning length $t$, and compare it with the greedy decoding algorithm to emphasize when it is important to use the sample-then-select method.

Our main result to this end is the following two theorems.
The first result is regarding the regime where we have sufficiently many in-context data $n$,
with proof in Appendix~\ref{subsec:_proof_th_sc_sufficient_n}.
\begin{theorem}[Perfect recovery probability with sufficient in-context examples]
\label{th:sc_sufficient_n}
    Suppose that $n \geq (6k + 3\sigma_{\epsilon})^4$, then the overall recovery probability of greedy decoding and majority vote are lower bounded as following:
    \begin{itemize}[ leftmargin=*]
    \item \textbf{Greedy decoding:} for any reasoning length $t\geq 1$, $\mathbb{P}\big(\wb^{\mathtt{greedy}}_t = \wb^{\ast}\big) \geq 1 - \delta(n)$
  \item \textbf{Majority vote:} for any reasoning length $t\geq 1$ and sampling number $N\geq 1$, it holds that 
        \begin{align}
            \mathbb{P}\left(\wb_{t, N}^{\mathtt{mv}} = \wb^{\ast}\right)  \geq \big(1-\delta(n)\big)\cdot \big(1 - |\cW| \cdot e^{-N\Delta_t^2/2}  \big).
        \end{align}
    \end{itemize}
    Here $\delta(n) = 2d(d+2)\cdot \exp(-c\cdot  n^{1/2})$ for some absolute constant $c>0$, and for any $t\geq 1$, $\Delta_t$ satisfies that 
    \begin{align}
    \Delta_t \geq \frac{p_{\mathtt{trans}}}{p_{\mathtt{trans}} + 1 - p_{\mathtt{recurr}}} \left(1\!-\!  \left(p_{\mathtt{recurr}} \!-\! p_{\mathtt{trans}}\right)^{t-1} \right),
    \end{align}
    where the quantities $p_{\mathtt{trans}}, p_{\mathtt{recurr}}\in(0,1)$ are defined as
    \begin{align}
        p_{\mathtt{trans}} := \left(1 - \frac{2k + \sigma_{\epsilon}}{n^{1/4} -\left(2k + \sigma_{\epsilon}\right)}\right)\cdot  \frac{1}{d^k}, \quad
        p_{\mathtt{recurr}} := \left(1 - \frac{\sigma_{\epsilon}}{n^{1/4} - \sigma_{\epsilon}}\right)\cdot  \left(\frac{n^{1/4} - \sigma_{\epsilon}}{n^{1/4} - \sigma_{\epsilon} + d \sigma_{\epsilon}}\right)^{k}.
    \end{align}
\end{theorem}
Theorem~\ref{th:sc_sufficient_n}  establishes lower bounds on the recovery probability for both greedy decoding and majority vote. 
The recovery probability improves exponentially with the number of in-context examples.
For majority vote, since $0< \Delta_t <1$ for all $t \geq 1$, as with sufficiently many number of sampling paths ($N \rightarrow \infty$), we have $\mathbb{P}\left(\wb_{t, \infty}^{\mathtt{mv}} = \wb^{\ast}\right) \geq 1-\delta$, which matches that of greedy decoding $ \mathbb{P}(\wb^{\mathtt{greedy}}_t = \wb^{\ast}) $, and both algorithms can achieve perfect accuracy given sufficient in-context examples $n$. 
Moreover, we remark that $p_{\mathtt{recurr}} > p_{\mathtt{trans}}$ since it holds that $(
n^{1/4} - \sigma_{\epsilon})(n^{1/4} - \sigma_{\epsilon} + d \sigma_{\epsilon})^{-1}>d^{-1}$
for sufficiently many in-context examples $n> (3\sigma_{\epsilon})^4$. 
When $\sigma_{\epsilon} = 0$, we have $p_{\mathtt{recurr}} = 1$ and $p_{\mathtt{trans}} > 1/2d^k$, ensuring that $\Delta_t$ converges to 1 as $t \rightarrow \infty$.


The theorem for sufficient in-context data does not highlight the advantage of majority vote in terms of recovery probability. However, real-world applications and our experiments show that majority vote is more accurate and robust with limited in-context data. We present our second main theorem to analyze this scenario, considering the case with only one in-context example ($n = 1$ and $k = 1$). Although simplified, this case offers valuable insights into the robustness of majority vote.

\begin{theorem}[Majority vote outperforms greedy decoding in the case of limited in-context examples] \label{th:sc_limited_n} 

Consider the case where $n = k = 1, \sigma_{\epsilon} = 0$, and denote the in-context example as $(\xb, \xb^\top \wb^{\ast} )$.
We have the following results.
\begin{itemize}[leftmargin=*,nosep]
    \item Greedy decoding: for any reasoning length $t\geq 1$, 
    \begin{align}
        \mathbb{P}\big(\wb^{\mathtt{greedy}}_t = \wb^{\ast}\big) \leq \frac{1}{2^{d-1}} + \frac2d.
    \end{align}
    \item Majority vote: there exists a $\zeta > 0$ such that for reasoning steps $t\geq 2\log2/\log(1 - \zeta)$, sampling number $N\geq 1$,
    \begin{align}
        \mathbb{P}\left(\wb_{t, N}^{\mathtt{mv}} = \wb^{\ast}\right) \geq 1 - \frac{1}{2^{d-1}}.
    \end{align}
\end{itemize}
\end{theorem}


Theorem~\ref{th:sc_limited_n}, detailed with proof in Appendix~\ref{subsec:_proof_th_sc_limited_n}, highlights a key difference between majority vote and greedy decoding with limited in-context examples. As shown in numerical experiments, greedy decoding frequently gets stuck in cyclic state transitions, failing to reach the optimal state $\wb^*$. In contrast, majority vote explores the state space more effectively, enabling a high probability of converging to $\wb^*$ even in constrained scenarios, as shown in numerical experiments in Section~\ref{subsec:_numerical}.



%% file: icml2025/sections/experiments.tex
\section{Experiments}
\subsection{Numerical Results for In-Context Linear Regression}\label{subsec:_numerical}
\begin{figure*}
\vskip -.1in
\centering
\begin{subfigure}
  \centering
  \includegraphics[width=0.95\textwidth]{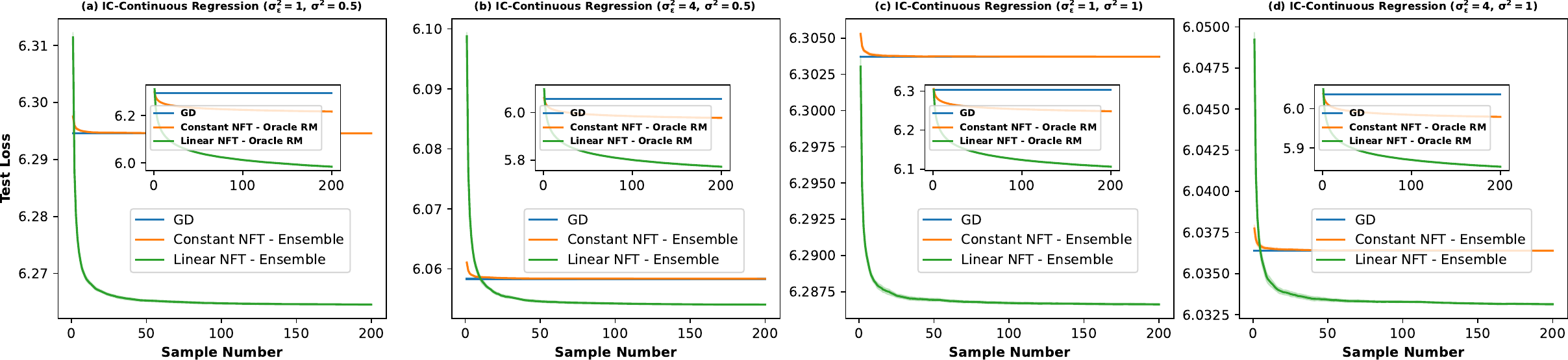}
\end{subfigure}
\begin{subfigure}
  \centering
  \includegraphics[width=0.95\textwidth]{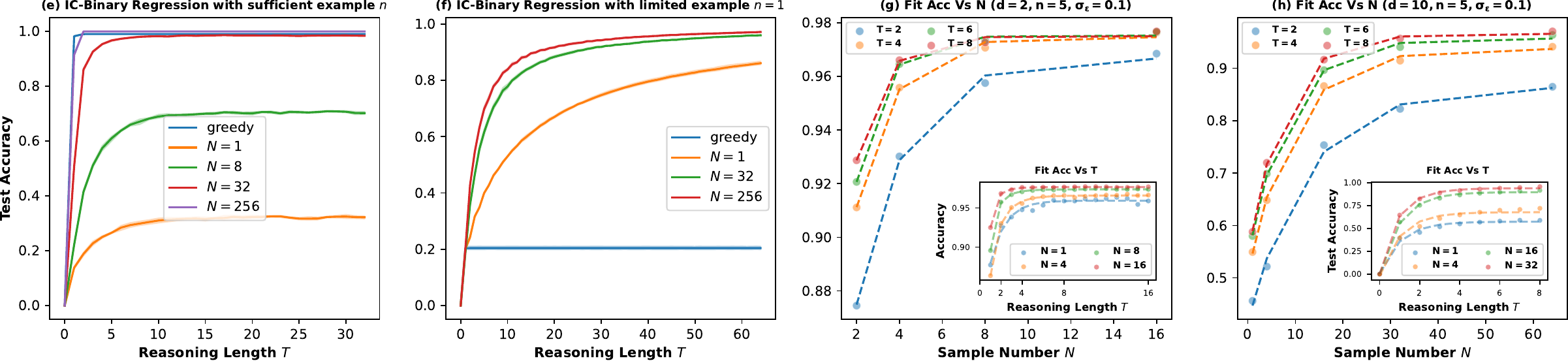}
\end{subfigure}
\vskip -.1in
\caption{Numerical experiments on in-context linear regression with continuous coefficients  (\textit{a-d}) and binary coefficients (\textit{e-h}).}
\label{fig:add_exp}
\vskip -.2in
\end{figure*}

Here, we validate our theoretical findings through numerical experiments. For the continuous case, we examine the effects of varying $\sigma_{\epsilon}$ and $\sigma$. Our results demonstrate that with ensemble aggregation, constant NFT provides no performance improvement, while linear NFT reduces test loss given sufficient sample size, confirming Corollary~\ref{cor:_average_continuous}. Furthermore, when decoding with a reward model, even constant NFT yields consistent performance improvements as sample numbers increase.

For the binary sparse coefficient case, we observe from Fig~\ref{fig:add_exp} (e) that with sufficient examples, both greedy decoding and majority voting achieve perfect accuracy, supporting Theorem~\ref{th:sc_sufficient_n}. From Fig~\ref{fig:add_exp} (f) we find that when setting $n=1$ and $d = 10, \sigma_{\epsilon} = 0$,  with sufficiently large reasoning length $T$, majority voting achieves high accuracy, while greedy search maintains approximately $2/d = 0.2$ accuracy, consistent with Theorem \ref{th:sc_limited_n}. We fit the relationship between accuracy $\mathtt{Acc}$ and sample number $N$ using $\mathtt{Acc} = \alpha_{T} - \beta_{T} e^{-\nu_T N}$ for given $T$. The results, shown in Fig~\ref{fig:add_exp} (g) and (h), not only validate Theorem~\ref{th:sc_maj_vote} but also suggest practical applications for real-world LLM inference.
\subsection{Insights for LLM Inference}

Our theoretical analysis reveals two critical terms $\mathcal{O}(e^{-\Delta_T^2 N/2})$ and $\mathcal{O}(e^{-\mu T})$ for the overall accuracy $\mathtt{Acc}(T,N)$ and probability gap $\Delta_T$. These findings can provide valuable insights into real-world LLM inference.

To begin, we can observe that $\Delta_T$ changes with the number of reasoning steps $T$ in $\mathcal{O}(e^{-\mu T})$. This can be described as:
\begin{align}\label{eq:fit_p_n}
\Delta_T &\approx \gamma - \kappa e^{-\mu T}.
\end{align}
Specifically, for sampling number of $N = 1$, here we \textit{assume} we can directly express the overall accuracy as :
\begin{align}\label{eq:fit_p_n_1}
    \mathtt{Acc}(T,1) &\approx \gamma' - \kappa' e^{-\mu T}.
\end{align}
Note that Eq~\eqref{eq:fit_p_n_1} and \eqref{eq:fit_p_n} shares the same $\mu$. To predict the final accuracy for a given sampling number $N$, here we introduce two additional parameters $(\alpha_{(T,N)}, \beta_{(T,N)})$ and formulate $\mathtt{Acc}(T,N)$ as:
\begin{align} \label{eq:fit_acc_t_n}
\mathtt{Acc}(T,N) &\approx \alpha_{(T,N)} - \beta_{(T,N)} e^{-\Delta_T^2 N/2}.
\end{align}
To effectively fit Eq~\eqref{eq:fit_p_n} - \eqref{eq:fit_acc_t_n}, based on the results on Fig \ref{fig:add_exp} (g) and (h), we further claim two conjectures:
\begin{itemize}[nosep,leftmargin=*]
    \item When $T$ is fixed, then  Eq~\ref{eq:fit_acc_t_n} can be approximated by:
        \begin{align}\label{eq:fit_acc_fix_t}
        \mathtt{Acc}(T,N) &\approx \alpha_{T} - \beta_{T} e^{-\Delta_T^2 N/2}.
        \end{align}
    \item When $N$ is fixed, then Eq~\ref{eq:fit_acc_t_n} can be approximated by:
        \begin{align}\label{eq:fit_acc_fix_n}
        \mathtt{Acc}(T,N) &\approx \alpha_{N} - \beta_{N} e^{-\Delta_T^2 N/2}.
        \end{align}
\end{itemize}

This analysis enables us to predict model's high test-time computation performance using data from low-computation, resulting our Low-Cost-to-High Prediction Algorithm \ref{alg:predicte_acc}, we validate our algorithm on GSM8K \citep{cobbe2021training} and a subset of MATH \citep{hendrycks2021measuring},  details can be found in \autoref{appdix:exp_details}. 
Fig \ref{fig:gsm8k_math_fit} demonstrates that our algorithm successfully predicts model performance at high computational costs using only data from settings with relatively low reasoning tokens $T$ or sampling numbers $N$.
\begin{figure}[!t]
\begin{center}
\centerline{\includegraphics[width=0.8\columnwidth]{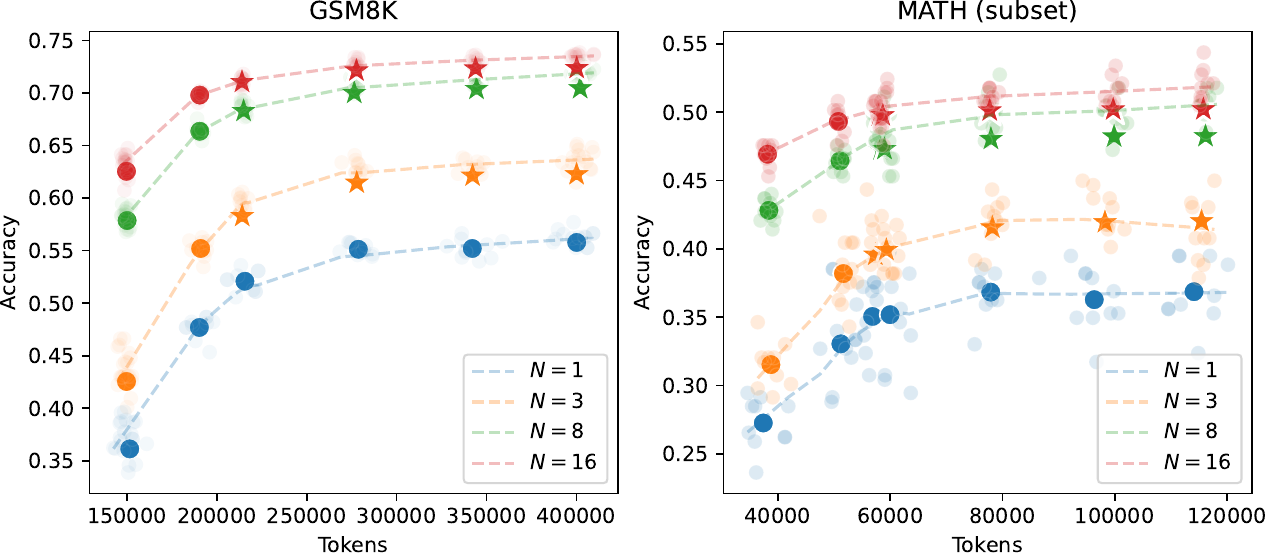}}
\vskip -.1in
\caption{Utilizing data with low computational costs to forecast results for high computational costs, where $\bigstar$ denotes predicted results and $\newmoon$ denotes the data utilized.}
\label{fig:gsm8k_math_fit}
\end{center}
\vskip -0.3in
\end{figure}

%% file: icml2025/sections/conclusions.tex
\section{Conclusions and Limitations}
This paper makes the initial step toward bridging the gap between practical language model test-time computing techniques with sampling and theoretical transformer analysis by incorporating randomness into the decoding process. 
We study the task of in-context linear regression with continuous/binary coefficients and provide a detailed analysis of widely adopted inference techniques, offering new insights into inference behaviors in real-world language models. 
Potential future works include analyzing other types of sampling algorithms and reasoning methods. Also it remains open to rigorously analyze the benefits of BoN method and its variants (with respect to different reward models) that we experimentally verified to be effective.

%% file: icml2025/appendix/experiments.tex
\section{Experiment Details}\label{appdix:exp_details}
\subsection{Experiment Settings}
\paragraph{Details for Figure~\ref{fig:gsm8k_sparse_comp}:} We evaluate real-world LLM using the GSM8K dataset \citep{cobbe2021training}, employing SGLang \citep{zheng2024sglang} as our inference framework; and use synthetic data with our theoretical framework to simulate practical decoding procedures. The experimental configurations are as follows:
\begin{itemize}

\item \textbf{LLM Performance on GSM8K with Varying Sample Number}: We employ Llama3.1-8b \citep{dubey2024llama} with an 8-shot chain-of-thought prompt following \citep{wei2022chain}. For each question, we generate 256 potential answers using decoding temperature of $1.0$. We implement an oracle reward model that perfectly validates answer correctness, and set the temperature to 0.0 for greedy search.

\item \textbf{LLM Performance on GSM8K with Varying Reasoning Lengths}: Using Llama3.1-8b-instruct, we analyze performance across different reasoning lengths, defined as the token consumption per inference call. Following \citep{zhang_o1_inference_scaling_laws}, we incorporate token budgets into the prompts to constrain the model's responses. For each prompt, we generate 64 potential answers and create $10$ random permutations of these answers. We define the reasoning length $T$ as the sum of token consumption across all prompts, and for multiple samples ($N > 1$), we average the token counts over $N$. The accuracy-tokens curves are plotted using transparent scattered points for individual permutations and fitted with trend lines. The prompt templates are provided in \ref{appdix:prompt}.

\item \textbf{IC-Linear Regression with Continuous Coefficients}: We configure the parameters as $n = 36, d = 72, \eta = 1\times 10^{-3}, \sigma_\epsilon^2 = 1,\sigma^2 = 4$, and present results at gradient descent iterations $t = 950$.

\item \textbf{IC-Linear Regression with Binary Coefficients}: We set the parameters to $n = 4, k = 1, d = 48, \eta = \frac{1}{4}, \sigma_\epsilon^2 = 0.25$.
\end{itemize}

\paragraph{Details for Figure~\ref{fig:add_exp}:} we conduct numerical experiments on in-context linear regression with continuous coefficients (\textit{above a-d}) and binary coefficients (\textit{below e-h}), each setting we repeat $5$ times, details are as follows:
\begin{itemize}

\item \textbf{Continuous case}: we set the parameters to $d = 72, n = 36, \eta = 10^{-3}$, and present results at gradient descent iterations $t = 950$.

\item \textbf{Binary case}: In \autoref{fig:add_exp} (e): we set $n = 40, k = 2, d = 30, \eta = \frac{1}{40}, \sigma_{\epsilon} = 0.1$; in (f): we set $n = 1, k = 1, d = 10, \eta = 1, \sigma_{\epsilon} = 0$; in (g): we set $n = 1, k = 1, d = 2, \eta = 1, \sigma_{\epsilon} = 0.1$; in (h): we set $n = 5, k = 1, d = 10, \eta = 1, \sigma_{\epsilon} = 0.1$.

\item \textbf{Fitting accuracy with varying reasoning length $T$}: for $N = 1$, we fit the curve with 
    \begin{equation*} \begin{aligned}
        \mathtt{Acc}(T,1) &\approx \alpha_1 - \beta_1 e^{-\mu_1 T},
    \end{aligned} \end{equation*}
    for $N > 1$, we first approximate $\Delta_T \approx \mathtt{Acc}(T,1) \approx \alpha_1 - \beta_1 e^{-\nu_1 T}$, where $(\alpha_1,\beta_1,\nu_1)$ are obtained in case $N=1$, then fit curve with
    \begin{equation*} \begin{aligned}
        \mathtt{Acc}(T,N) &\approx \alpha_N - \beta_N e^{-\mu_N \Delta_T^2}.
    \end{aligned} \end{equation*}
\end{itemize}

\paragraph{Details for Figure~\ref{fig:gsm8k_math_fit}:} We conduct experiments using GSM8K and a curated subset of the MATH dataset \citep{hendrycks2021measuring}, details are as follows:

\begin{itemize}

\item \textbf{MATH Dataset Subset}: We filter the MATH to extract problems at level 1 with integer answers, yielding a subset of 309 problems.

\item We maintain consistent experimental settings with the GSM8K reasoning length evaluation as in Figure~\ref{fig:gsm8k_sparse_comp}, utilizing Llama3.1-8b-instruct with a decoding temperature of $1.0$. To facilitate the fitting process in Algorithm~\ref{alg:predicte_acc}, we apply a scaling factor of $\frac{1}{10^5}$ to the token count, $T' = \frac{T}{10^5}$.
\end{itemize}
\subsection{Low-Cost-to-High Prediction algorithm}

\label{subsec:predicte_acc}
\begin{algorithm}[H]
\caption{Low-Cost-to-High Prediction algorithm}\label{alg:predicte_acc}
\textbf{Part 1:} Obtain $(\gamma , \kappa, \mu)$ in Eq~\ref{eq:fit_p_n}

\begin{algorithmic}[1]
  \STATE \textbf{Input:} Data at varying cost $\{\mathtt{Acc}^{(e)}(T_i,N_j)\}$,$T_i \in \cT^{(e)}, N_j \in \cN^{(e)}$;
  \STATE $(\gamma',\kappa',\mu) \leftarrow$ Fit Eq~\ref{eq:fit_p_n_1} with $\{\mathtt{Acc}^{(e)}(T_i,1)\}_{\cT^{(e)}}$ 
  \STATE $(\alpha_{T_1},\beta_{T_1} ,\Delta_{T_1}) \leftarrow$ Fit Eq~\ref{eq:fit_acc_fix_t} with $\{\mathtt{Acc}^{(e)}(T_1,N_j)\}$ 
  \STATE $(\alpha_{T_2},\beta_{T_2} ,\Delta_{T_2}) \leftarrow$ Fit Eq~\ref{eq:fit_acc_fix_t} with $\{\mathtt{Acc}^{(e)}(T_2,N_j)\}$ 
  \STATE $(\gamma , \kappa)\leftarrow $ Fit Eq~\ref{eq:fit_p_n} with $\{(\Delta_{T_0}, \mu), (\Delta_{T_1}, \mu)\}$
  \STATE \textbf{Return} $(\gamma , \kappa, \mu)$
\end{algorithmic}

\textbf{Part 2:} Predict accuracy with $(\gamma , \kappa, \mu)$ and low cost data

\begin{algorithmic}[1]
  \STATE \textbf{Input:} $(\gamma , \kappa, \mu)$ in Eq~\ref{eq:fit_p_n},$\cD_N = \{\mathtt{Acc}^{(e)}(T_1,N), \mathtt{Acc}^{(e)}(T_2,N)\}$;
  \STATE $\Delta_{T_i} \leftarrow \gamma - \kappa e^{-\mu T_i}, i = 1,2$ \hfill\COMMENT{//Eq~\ref{eq:fit_p_n}}
  \STATE $(\alpha_{N}, \beta_{N}) \leftarrow$ Fit Eq~\ref{eq:fit_acc_fix_n} with two data points: $\{(\mathtt{Acc}^{(e)}(T_1,N), P_{T_1}), (\mathtt{Acc}^{(e)}(T_2,N), P_{T_2})\}$ 
  \STATE Use Eq~\ref{eq:fit_p_n},Eq~\ref{eq:fit_acc_fix_n} with obtained $(\gamma , \kappa, \mu)$ and $(\alpha_{N}, \beta_{N})$ to predict data with varying $T$.
\end{algorithmic}
\end{algorithm}

The core ideal of Algorithm~\ref{alg:predicte_acc} is to first determine $(\gamma, \kappa, \mu)$ in Equation \ref{eq:fit_p_n}. Subsequently, we can compute $\Delta_T$ and Equation \ref{eq:fit_acc_fix_t} using two additional parameters $\alpha_{N}, \beta_{N}$, obtainable from only two data points. Notably, since we use ${\mathtt{Acc}^{(e)}(T_0,N_j)}$ and ${\mathtt{Acc}^{(e)}(T_1,N_j)}$ during the initial parameter estimation (Algorithm~\ref{alg:predicte_acc} Part 1, lines 3-4), no additional data is required for subsequent predictions in part 2.

%% file: icml2025/appendix/transformer.tex
\section{Proofs for Section~\ref{sec:_test_time_computing_framework}}

\subsection{Proof of Proposition~\ref{prop:_construction}}\label{subsec:_proof_prop_construction} 

\begin{proof}[Proof of Proposition~\ref{prop:_construction}]
    The proof is based on the proof of Theorem 3.2 of \cite{huang2025transformers}. 
    We take the desired parameter $\theta_{\mathtt{GD}} = \{\mathbf{V}_{\mathtt{GD}}, \Wb_{\mathtt{GD}}\}$ as following, 
    \begin{align}
        \mathbf{V}_{\mathtt{GD}} := \left(\begin{matrix}
        \boldsymbol{0}&\boldsymbol{0} & \boldsymbol{0} &\boldsymbol{0} \\
        \boldsymbol{0}&0&\boldsymbol{0}& 0\\
        -\eta\cdot \Ib_d&\boldsymbol{0}&\boldsymbol{0}&\boldsymbol{0}\\
        \boldsymbol{0}&0&\boldsymbol{0}&0
    \end{matrix}\right),\quad \mathbf{W}_{\mathtt{GD}} := \left(\begin{matrix}
        \boldsymbol{0}&\boldsymbol{0} & \Ib_d &\boldsymbol{0} \\
        \boldsymbol{0}&0&\boldsymbol{0}& -1\\
        \boldsymbol{0}&\boldsymbol{0}&\boldsymbol{0}&\boldsymbol{0}\\
        \boldsymbol{0}&0&\boldsymbol{0}&0
    \end{matrix}\right),
    \end{align}
    Then one can check that when inputting $\Hb_{\ell}$ in the form of
    \begin{align}
        \Hb_{\ell} = \left(\begin{matrix}
        \xb_1&\cdots & \xb_n &\boldsymbol{0} &\cdots & \boldsymbol{0}\\
        y_1&\cdots&y_n& 0 &\cdots &0\\
        \boldsymbol{0}&\cdots&\boldsymbol{0}&\wb_0&\cdots &\wb_{\ell}\\
        0&\cdots&0&1&\cdots & 1
    \end{matrix}\right),
    \end{align}
    the output embedding of the transformer at the last token is given by 
    \begin{align}
        (\widetilde{\Hb}_\ell)_{:,-1} = \left(\begin{matrix}
            \boldsymbol{0}\\
            0\\
            \widetilde{\wb}_{\ell}\\
            1
        \end{matrix}\right),\quad \widetilde{\wb}_{\ell} = \wb_{\ell} - \frac{\eta}{n}\cdot \Xb^\top (\Xb \wb_{\ell} - \yb).
    \end{align}
    Thus if we take the sampling algorithm $\mathtt{Sampling\_Alg}(\cdot)$ satisfying the form of
    \begin{align}
        \mathtt{Sampling\_Alg}(\hb) = \delta_{\boldsymbol{0}}(\cdot)\otimes \delta_0(\cdot)\otimes p\left(\cdot\middle|(\hb)_{d+2:2d+1}\right) \otimes \delta_1(\cdot),
    \end{align}
    for some conditional distribution $p:\RR^d\mapsto\cP(\RR^d)$, then the embedding of the next token would be 
    \begin{align}
        \hb_{\ell+1} = \left(\begin{matrix}
            \boldsymbol{0}\\
            0\\
            \wb_{\ell}\\
            1
        \end{matrix}\right),\quad \wb_{\ell+1} \sim p\left(\,\cdot\,\middle|\wb_{\ell} - \frac{\eta}{n}\cdot \Xb^\top\big(\Xb \wb_{\ell}-  \yb\big)\right),
    \end{align}
    by Definition~\ref{def:_inference_mechanism}. 
    Iterating the above argument from $\ell=0$ to $t-1$ completes the proof of Proposition~\ref{prop:_construction}.
\end{proof}

\subsection{Special Case: Vanilla Multi-step Grandient Descent with CoT}\label{subsec:_special_case}

One special case of Proposition~\ref{prop:_construction} is a transformer that explicitly performs standard multi-step gradient descent (GD) \citep{huang2025transformers}, i.e., $p(\cdot|x) = \delta_x(\cdot)$, so that the final prediction of the regression coefficient after $t$ reasoning steps is given by
\begin{align}
    \wb_{\mathtt{GD}} :=(\Hb_t)_{d+2:2d+1, n+t} = \left(\Ib_d - \left(\Ib_d - \frac{\eta}{n}\cdot \Xb^\top\Xb\right)^t\right)\Xb^\top (\Xb\Xb^\top)^{-1}\yb.\label{eq:_w_gd}
\end{align}
We note that \cite{huang2025transformers} considers transformer CoT reasoning for in-context-linear regression with \emph{noiseless} labels, but here we allow the existence of label noise.

%% file: icml2025/sections/continuous_continued.tex
\section{Theoretical Analysis in Section~\ref{sec:continuous_analysis} Continued}\label{subsec:_continuous_continued}

\begin{theorem}[Excess risk of vanilla multi-step GD with CoT: general covariance matrix]\label{thm:_gd_continuous}
    Under Assumption~\ref{ass:_data_distribution}, taking the step size $\eta\leq \|\Hb\|_2^{-1}$ and CoT length $t$, with probability at least $1-1/\mathrm{poly}(n)$, it holds that 
    \begin{align}
        \mathbb{E}_{\boldsymbol{\epsilon},\wb^{\ast}}\left[\mathcal{E}(\wb_{\mathtt{GD}})\right]\lesssim \omega^2\cdot\Bigg(\frac{\widetilde{\lambda}^2}{n^2}\cdot \sum_{1\leq i\leq k^{\ast}} \frac{1}{\lambda_i}+ \sum_{k^{\ast}<i\leq d}\lambda_i\Bigg) + \sigma^2_{\epsilon}\cdot\Bigg(\frac{k^{\ast}}{n} + \frac{n}{\widetilde{\lambda}^2}\cdot\sum_{k^{\ast}<i\leq d}\lambda_i^2\Bigg),
    \end{align}
    where the quantities are as follows 
    \begin{align}
        k^{\ast} := \min\big\{k: n\lambda_{k+1}\leq\frac{n}{\eta t} + \sum_{k<i\leq d} \lambda_i\big\},\quad 
        \widetilde{\lambda}:=\frac{n}{\eta t} + \sum_{k^{\ast} < i\leq d}\lambda_i.
        \label{eq:_k_lambda_vanilla}
    \end{align}
\end{theorem}

\begin{proof}[Proof of Theorem~\ref{thm:_gd_continuous}]
    Please refer to Appendix~\ref{subsec:_proof_thm_gd_continuous} for a proof of Theorem~\ref{thm:_gd_continuous}.
\end{proof}

\begin{theorem}[Excess risk of noisy multi-step noisy GD with CoT and ensembling]\label{thm:_average_continuous}Under Assumption~\ref{ass:_data_distribution}, taking the step size $\eta\leq \|\Hb\|_2^{-1}$ and CoT length $t$, we have the following risk bounds for $\wb_{\mathtt{avg}}$.
\begin{enumerate}[left=0mm]
\item Constant noise transformation function (Example~\ref{exp:_constant}): with probability at least $1-1/\mathrm{poly}(n)$, 
    \begin{align}
        \mathbb{E}_{\boldsymbol{\epsilon},\wb^{\ast}}\left[\mathcal{E}(\wb_{\mathtt{GD}})\right]&\lesssim \omega^2\cdot\Bigg(\frac{\widetilde{\lambda}^2}{n^2}\cdot \sum_{1\leq i\leq k^{\ast}} \frac{1}{\lambda_i} + \sum_{k^{\ast}<i\leq d}\lambda_i\Bigg) + \frac{\vartheta_{n,t}}{N},
    \end{align}
    where the quantities $k^{\ast}$ and $\widetilde{\lambda}$ are defined as the same as \eqref{eq:_k_lambda_vanilla}, and $\vartheta_t$ is defined as 
    \begin{align}
        \vartheta_{n,t}:=\sigma^2d \cdot \left(t  \cdot \sqrt{\frac{   r(\Hb) \vee \log (\mathrm{poly(n)})}{n}} +  \frac{1}{\eta}\right),\label{eq:_vartheta}
    \end{align}
    with $r(\Hb) = \tr(\Hb) / \norm{\Hb}_2$ being the effective rank of $\Hb$.
\item Linear noise transformation function (Example~\ref{exp:_linear}): taking the noise variance $\sigma^2\asymp d^{-1}$ and the reasoning path length $t>\sigma^{-2}\cdot \log 2$, with probability at least $1-1/\mathrm{poly}(n)$, 
    \begin{align}
        \mathbb{E}_{\boldsymbol{\epsilon},\wb^{\ast},\bxi}\left[\mathcal{E}(\wb_{\mathtt{avg}})\right]\lesssim \omega^2\!\cdot\!\Bigg(\frac{(\widetilde{\lambda}^{\mathrm{Bias}})^2}{n^2}\cdot \!\!\!\!\sum_{1\leq i\leq k^{\ast}_{\mathrm{Bias}}} \!\frac{1}{\lambda_i} + \!\!\!\!\sum_{k^{\ast}_{\mathrm{Bias}}<i\leq d}\!\lambda_i\Bigg) + \sigma^2_{\epsilon}\!\cdot\!\Bigg(\frac{k^{\ast}_{\mathrm{Var}}}{n} + \frac{n}{(\widetilde{\lambda}^{\mathrm{Var}})^2}\cdot\!\!\!\!\sum_{k^{\ast}_{\mathrm{Var}}<i\leq d}\!\lambda_i^2\Bigg) + \frac{\varsigma_n}{N},
    \end{align}
    where the quantities $\widetilde{\lambda}^{\mathrm{Bias}}$, $\widetilde{\lambda}^{\mathrm{Var}}$, $k^{\ast}_{\mathrm{Bias}}$, and $k^{\ast}_{\mathrm{Var}}$ are defined as following respectively,
    \begin{align}
        k^{\ast}_{(\diamondsuit)} := \min\Bigg\{k\in[d]: n\lambda_{k+1}\leq\widetilde{\lambda}_{\mathrm{effect}}^{(\diamondsuit)} + \!\!\sum_{k<i\leq d} \lambda_i\Bigg\},\quad \widetilde{\lambda}^{(\diamondsuit)}:=\widetilde{\lambda}_{\mathrm{effect}}^{(\diamondsuit)}   +\!\!\! \sum_{k^{\ast} < i\leq d}\lambda_i, \quad \mathrm{for}\,\,(\diamondsuit)\in\{\mathrm{Bias}, \mathrm{Var}\},
    \end{align}
    with $\widetilde{\lambda}_{\mathrm{effect}}^{\mathrm{Bias}} $ and $\widetilde{\lambda}_{\mathrm{effect}}^{\mathrm{Var}} $ defined as,
    \begin{align}
        \widetilde{\lambda}_{\mathrm{effect}}^{\mathrm{Bias}}  := \frac{n}{\eta}\cdot \left(\frac{2}{t} + \frac{\sigma^2}{1-\sigma^2}\left(1+\frac{2}{t}\right)\right),\quad \widetilde{\lambda}_{\mathrm{effect}}^{\mathrm{Var}}  := \frac{\sigma^2n}{(1-\sigma^2)\eta},
    \end{align}
    and the quantity $\varsigma_n$, is given by 
    \begin{align}
        \varsigma_{n}:= \left( \frac{\eta\sigma_{\epsilon}^2d}{n\sigma^2}\cdot \Tr(\Hb)  + \omega^2\right)\cdot \|\Hb\|_2.\label{eq:_varsigma}
    \end{align}
    \end{enumerate}
\end{theorem}

\begin{proof}[Proof of Theorem~\ref{thm:_average_continuous}]
    Please refer to Appendix~\ref{subsec:_proof_thm_average_continuous} for a proof of Theorem~\ref{thm:_average_continuous}.
\end{proof}

\begin{corollary}[Theorem~\ref{cor:_average_continuous} restated]\label{cor:_average_continuous_restate}
     Under the same assumptions and setups as in Theorem~\ref{thm:_average_continuous},  additionally assuming that the spectrum of $\Hb$ satisfies polynomially decaying, i.e., $\lambda_i = i^{-(r+1)}$ for some constant $r\geq 0$,  we have the following results.
     \begin{enumerate}[left=0mm]
         \item Constant noise transformation function (Example~\ref{exp:_constant}): taking the  reasoning path length $t\lesssim \eta (r+1)^{(r+2)/2}n^{(r+1)/2}$ and the sampling path number 
         \begin{align}
             N\geq N_{\mathrm{c}}:=\left(\sigma^2d\cdot \left(t  \cdot \sqrt{\frac{   r(\Hb) \vee \log (\mathrm{poly(n)})}{n}} +  \frac{1}{\eta}\right)\right)\cdot \left(\omega^2\cdot\left(\frac{1}{t\eta}\right)^{\frac{r}{r+1}} + \frac{\sigma_{\epsilon}^2}{n}\cdot (t\eta)^{\frac{1}{r+1}}\right)^{-1},\label{eq:_n_c}
         \end{align}
         then with probability at least $1-1/\mathrm{poly}(n)$, 
             \begin{align}
        \mathbb{E}\left[\mathcal{E}(\wb_{\mathtt{avg}})\right]\lesssim \omega^2\cdot\left(\frac{1}{t\eta}\right)^{\frac{r}{r+1}} + \frac{\sigma_{\epsilon}^2}{n}\cdot (t\eta)^{\frac{1}{r+1}}.
    \end{align}
         \item Linear noise transformation function (Example~\ref{exp:_linear}): taking the noise variance $\sigma^2\asymp d^{-1}$, the reasoning path length $\sigma^{-2}\cdot \log 2<t$, and the sampling path number 
         \begin{align}
             N\geq N_{l}&:= \left( \omega^2+\frac{\eta\sigma_{\epsilon}^2d\cdot \Tr(\Hb)}{n\sigma^2} \right)\cdot \|\Hb\|_2\cdot \left(\omega^2\cdot\left(\frac{\sigma^2}{\eta\cdot (1-\sigma^2)}\right)^{\frac{r}{r+1}} + \frac{\sigma_{\epsilon}^2}{n}\cdot \left(\frac{\eta\cdot (1-\sigma^2)}{\sigma^2}\right)^{\frac{1}{r+1}}\right)^{-1}\\
             &\asymp \left(\omega^2+\frac{\sigma^2_{\epsilon}}{n}\cdot \eta d^{2}\right)\cdot \left(\omega^2 \cdot \left(\frac{1}{\eta d}\right)^{\frac{r}{r+1}} + \frac{\sigma_{\epsilon}^2}{n}\cdot (\eta d)^{\frac{1}{r+1}}\right)^{-1}\label{eq:_n_l}
         \end{align}
         then with probability at least $1-1/\mathrm{poly}(n)$, 
        \begin{align}
            \mathbb{E}\left[\mathcal{E}(\wb_{\mathtt{avg}})\right]\lesssim\omega^2\!\cdot\! \widetilde{\lambda}^{\frac{r}{r+1}} +   \frac{\sigma_{\epsilon}^2}{n}\cdot \left(\frac{\eta(1-\sigma^2)}{\sigma^2}\right)^{\frac{1}{r+1}},
    \end{align}
    where $\widetilde{\lambda}:= \eta^{-1}(2t^{-1}+\sigma^2(1+2t^{-1})/(1-\sigma^2))$.
     \end{enumerate}
     Here the expectation is taken with respect to $\boldsymbol{\epsilon}$, $\wb^{\ast}$, and the sampling noise $\bxi$ across different reasoning steps and paths.
\end{corollary}

\begin{remark}\label{rmk:_regime}
    Under the parameter regime of \eqref{eq:_regime}, i.e., 
    \begin{align}
        \omega\asymp 1, \quad \sigma_{\epsilon}\asymp 1, \quad n\asymp \eta d,\quad \sigma^{2}\asymp d^{-1},\label{eq:_regime_appendix}
    \end{align}
    we can obtain further simplifications of the above result. 
    Concretely, for the linear NFT setup, the number of sample paths needed is given by 
    \begin{align}
        N\geq N_l \asymp \big(\omega^2+\sigma_{\epsilon}^2d\big)\cdot \left(\big(\omega^2+\sigma_{\epsilon}^2\big)\cdot\left(\frac{1}{\eta d}\right)^{\frac{r}{r+1}}\right)^{-1}\asymp d^{\frac{2r+1}{r+1}},
    \end{align}
    and the excess risk bound is explicitly calculated by 
    \begin{align}
        \mathbb{E}_{\boldsymbol{\epsilon},\wb^{\ast},\bxi}\left[\mathcal{E}(\wb_{\mathtt{avg,linear}})\right] \lesssim \big(\omega^2+\sigma_{\epsilon}^2\big)\cdot \left(\eta d\right)^{-\frac{r}{r+1}}\asymp d^{-\frac{r}{r+1}}.
    \end{align}
    In contrast, we can also calculate that the risk bounds for either GD or ensemble with constant NFT is then given by 
    \begin{align}
        \mathbb{E}_{\boldsymbol{\epsilon},\wb^{\ast}}\left[\mathcal{E}(\wb_{\mathtt{GD}})\right], \mathbb{E}_{\boldsymbol{\epsilon},\wb^{\ast},\bxi}\left[\mathcal{E}(\wb_{\mathtt{avg,const}})\right] \lesssim \widetilde{t}^{\frac{1}{r+1}}\cdot \big(\omega^2+\sigma_{\epsilon}^2\big)\cdot \left(\eta d\right)^{-\frac{r}{r+1}}\asymp \widetilde{t}^{\frac{1}{r+1}}\cdot d^{-\frac{r}{r+1}}.
    \end{align}
    where $\widetilde{t} = \sigma^2\cdot t$ is the scaled reasoning length, satisfying $\widetilde{t}\lesssim d^{(r-1)/2}$.
\end{remark}

%% file: icml2025/appendix/continuous.tex
\section{Proofs for In-context Linear Regression with Continuous Coefficient (Section~\ref{sec:continuous_analysis})}

We denote the sample covariance matrix of the in-context data as $\boldsymbol{\Sigma} := n^{-1}\Xb^\top\Xb\in\RR^{d\times d}$, and we define the gram matrix of the in-context data as $\Ab := \Xb\Xb^\top\in\RR^{n\times n}$. 
Our results in this section depend on the following standard technical assumptions on the in-context data and task distributions.

\begin{assumption}[Data distribution]\label{ass:_data_distribution}
    We assume the following on the in-context data distribution $\cD_{\wb^{\ast}}$: 
    \begin{enumerate}[left=0mm]
        \item The columns of $\Hb^{-1/2}\xb$ are independent and $1$-subGaussian; 
        \item The labels are generated according to $y = \xb^\top\wb^{\ast} + \epsilon$, where the label noise $\epsilon$ is independent of $\xb$ and  satisfies $\mathbb{E}[\epsilon] = 0$ and $\mathbb{E}[\epsilon^2] = \sigma_{\epsilon}^2$ for some constant $\sigma_{\epsilon}>0$;
        \item The true coefficient $\wb^{\ast}$ follows the Gaussian prior, i.e., $\wb^{\ast}\sim \cN(\boldsymbol{0},\omega^2\cdot\Ib_d)$ for some constant $\omega>0$.
    \end{enumerate}
\end{assumption}

\subsection{Proof of Theorem~\ref{thm:_gd_continuous}}\label{subsec:_proof_thm_gd_continuous}

\begin{proof}[Proof of Theorem~\ref{thm:_gd_continuous}]
    This follows from the same arguments as in the proof of Theorem 4.3 in \cite{zou2022risk}.
    We refer the readers to their proofs for seek of simplicity.
\end{proof}

\subsection{Proof of Proposition~\ref{cor:_gd_continuous_polynomial_decay}}\label{subsec:proof_cor_gd_continuous_polynomial_decay}

\begin{proof}[Proof of Proposition~\ref{cor:_gd_continuous_polynomial_decay}] 
As a special case of Theorem~\ref{thm:_gd_continuous}, we begin by figuring out the optimal index $k^{\ast}$.
    We are going to prove that under the conditions in Proposition~\ref{cor:_gd_continuous_polynomial_decay}, the optimal index is given by 
    \begin{align}
        k^{\ast} = (\eta t)^{\frac{1}{r+1}} -1.
    \end{align}
    Notice that here without loss of generality we assume that the above quantity is an integer since otherwise we can twist $\eta$ (which is continuous) a little bit to make it an integer.
    And also we notice that the above $k^{\ast}\leq d$ due to our condition on $t$ in Proposition~\ref{cor:_gd_continuous_polynomial_decay}.
    To prove this, it suffices to check that the above $k^{\ast}$ is the smallest one satisfying the constraint in \eqref{eq:_k_lambda_vanilla}.
    To show it satisfies the constraint, consider 
    \begin{align}
        n\lambda_{k^{\ast}+1} = \frac{n}{(k^{\ast}+1)^{r+1}}= \frac{n}{\eta t}\leq \frac{n}{\eta t} + \sum_{k<i\leq d}\lambda_i.
    \end{align}
    To show that it is the smallest one satisfying the constraint, let's consider the other side of the inequality for $k^{\ast}-1$.
    We have the following calculations.
    On the one hand, we have
    \begin{align}
        n\lambda_{k^{\ast}} = \frac{n}{\left((\eta t)^{\frac{1}{r+1}} - 1\right)^{r+1}} = \frac{n}{\eta t}\cdot\frac{1}{\left(1 - (\eta t)^{-\frac{1}{r+1}}\right)^{r+1}} \geq \frac{n}{\eta t}\cdot \left(1+ (r+1)\cdot\left(\frac{1}{\eta t}\right)^{\frac{1}{r+1}} \right),\label{eq:_proof_cor_gd_1}
    \end{align}
    where the last inequality follows using $\log(1+x)\leq x$ and $\exp(x)\geq 1+x$ to obtain the following argument 
    \begin{align}
        \frac{1}{\left(1 - (\eta t)^{-\frac{1}{r+1}}\right)^{r+1}} = \exp\left(-(r+1) \log\left(1 - (\eta t)^{-\frac{1}{r+1}}\right)\right)\geq \exp\left((r+1)(\eta t)^{-\frac{1}{r+1}}\right) \geq 1+ (r+1)(\eta t)^{-\frac{1}{r+1}}.
    \end{align}
    On the other hand, we have that 
    \begin{align}
        \frac{n}{\eta t} + \sum_{k^{\ast}-1<i\leq d}\lambda_i \leq \frac{n}{\eta t} + \sum_{i>k^{\ast}-1}\frac{1}{i^{r+1}}\leq \frac{n}{\eta t} + \frac{1}{\left((\eta t)^{\frac{1}{r+1}} - 1\right)^r} \lesssim\frac{n}{\eta t} +\left(\frac{1}{\eta t}\right)^{\frac{r}{r+1}}.\label{eq:_proof_cor_gd_2}
    \end{align}
    Now to see that $k^{\ast}-1$ does not satisfies the constraint, in view of \eqref{eq:_proof_cor_gd_1} and \eqref{eq:_proof_cor_gd_2}, it boils down to show that
    \begin{align}
        \frac{n}{\eta t}\cdot \left(1+ (r+1)\cdot\left(\frac{1}{\eta t}\right)^{\frac{1}{r+1}} \right) \geq \frac{n}{\eta t} +\left(\frac{1}{\eta t}\right)^{\frac{r}{r+1}},
    \end{align}
    which is equivalent to restricting the reasoning path length $t$ satisfying $t\leq  \eta\cdot (r+1)^{\frac{r+1}{2}} \cdot n^{\frac{r+1}{2}}$. 
    According to our condition on the reasoning path length $t$ in Proposition~\ref{cor:_gd_continuous_polynomial_decay}, this requirement does hold, and thus $k^{\ast}-1$ does not satisfy the constraint. 
    Therefore we have proved that $k^{\ast} = (\eta t)^{\frac{1}{r+1}} - 1$.
    
    With the $k^{\ast}$ in hand, we can then follow the same arguments as in the proof of Corollary 4.5 in \cite{zou2022risk} to obtain the final result. 
    This completes the proof of Proposition~\ref{cor:_gd_continuous_polynomial_decay}.
\end{proof}

\subsection{Proof of Theorem~\ref{thm:_average_continuous}}\label{subsec:_proof_thm_average_continuous}

\subsubsection{Proof for Example~\ref{exp:_constant}} 
\begin{proof}[Proof of Theorem~\ref{thm:_average_continuous} for Example~\ref{exp:_constant}]
Under this setting, each reasoning path is generated though the following iteration: 
\begin{align}
    \wb_{ t+1}^{(j)}&=\wb_{ t}^{(j)}  - \frac{\eta}{n}\Xb^\top(\Xb \wb_t^{(j)} - \yb)  + \bxi_t^{(j)}. 
\end{align}
Based on this, we define the expected path $\wb_{t}^{\texttt{GD}(\eta; \Xb,\yb)}$ and the fluctuation $\Delta_t^{(j)}$ iteratively as
\begin{align}
    \wb_{t+1}^{\texttt{GD}(\eta; \Xb,\yb)}  &= \wb_{t}^{\texttt{GD}(\eta; \Xb,\yb)} - \frac{\eta}{n}\Xb^\top(\Xb \wb_t^{\texttt{GD}(\eta; \Xb,\yb)} - \yb), \label{eq:_proof_gd_continuous_0} \\
    \Delta_{t+1}^{(j)}&=  \wb_t^{(j)} - \wb_{t}^{\texttt{GD}(\eta; \Xb,\yb)} \\ 
    &= (\Ib -\eta\bSigma )\Delta_{t}^{(j)}  +  \bxi_t^{(j)}. 
\end{align}
By this characterization, we see that $\{\Delta_t^{(j)}\}_{j \le N}$ is a sequence of  iid zero-mean random variable for fixed $t$. 
This expectation-fluctuation decomposition allows us to recast the risk of the sample averaged output as 
    \begin{align}
        \cE(\wb_t^{\texttt{avg}}) =  \cE(\wb_t^{\texttt{GD}(\eta; \Xb,\yb)}) + N^{-1} \EE\big[\norm{\Delta_t^{(1)}}_{\Hb}^2\big]. 
        \label{eq:risk_decompose_const_0}
    \end{align}
In Theorem~\ref{thm:_gd_continuous}, we have characterized the average-case risk of the gradient descent, therefore it suffices to study the fluctuation of a single reasoning path. 
In the sequel, we drop the superscript $j$ for simplicity. 
Define $\Sbb_t = \EE[\Delta_t \Delta_t^\top]$, then we have that 
\begin{align}
    \Sbb_{t+1} &= (\Ib - \eta \bSigma) \Sbb_t (\Ib - \eta \bSigma)^\top + \sigma^2 \Ib \\ 
    &= \sum_{j=0}^t \sigma^2 (\Ib - \eta \bSigma)^{2j},
\end{align}
where the last identity holds because of the deterministic initialization $\Sbb_0 = 0$. 
Now we have that 
\begin{align}
\EE[ \norm{\Delta_t^{(j)}}_{\Hb}^2] &=  \dotp{\Sbb_t}{\bSigma} + |\dotp{\Sbb_t}{\Hb - \bSigma}| \\
&\le  \Tr \Big(\sum_{j=0}^{t-1} \sigma^2 (\Ib - \eta \bSigma)^{2j} \bSigma\Big)  +  \Tr(\Sbb_t) \cdot  \norm{ \Hb - \bSigma}_2. \label{eq:fluc_decompose_const_0}
\end{align}
For the first term above, we have that $\sum_{j=0}^t(1-\eta\lambda)^{2j}\lambda \le 1/\eta$ for $\lambda \in [0,1/\eta]$. 
For the second term , we have by \citet[Theorem 9]{koltchinskii2017concentration} that there exists an event with probability $1-\delta$ over the randomness of $\Xb$, on which it holds that 
\begin{align}
    \norm{\Hb -\bSigma}_2 \lesssim  \sqrt{\frac{r(\Hb) \vee \log (1/\delta)}{n}}, 
\end{align}
where $r(\Hb) = \Tr(\Hb) / \norm{\Hb}_2$ is the effective rank of $\Hb$. 
And we have the trivial upper bound that $\Tr(\Sbb_t) \le \sigma^2 d\cdot t$. 
Plugging them into \eqref{eq:risk_decompose_const_0} and~\eqref{eq:fluc_decompose_const_0}, we get that
\begin{align}
\cE(\wb_t^{\texttt{avg}}) &\le \cE(\wb_t^{\texttt{GD}(\eta; \Xb,\yb)}) + N^{-1} \dotp{\Sbb_t}{ \Hb} \\ 
&\le   \cE(\wb_t^{\texttt{GD}(\eta; \Xb,\yb)}) + \frac{\sigma^2d }{N}\Big(t  \cdot \sqrt{\frac{   r(\Hb) \vee \log (1/\delta)}{n}} +  \frac{1}{\eta}\Big).
\end{align}
This concludes the proof of the theorem. 
\end{proof}

\subsubsection{Proof for Example~\ref{exp:_linear}}

Now we give the proof of Theorem~\ref{thm:_average_continuous} for Example~\ref{exp:_linear}. 
The proof relies on the following key lemmas.

\begin{lemma}[Error decomposition]\label{lem:_average_continuous_decomposition}
    The difference between $\wb_{\mathtt{avg}}$ and the true coefficient $\wb^{\ast}$ can be decomposed as following, 
    \begin{align}
        \left\|\wb_{\mathtt{avg}} - \wb^{\ast}\right\|_{\Hb}^2 \leq  \mathrm{Bias} + \mathrm{Variance} + \mathrm{Fluctuation},
    \end{align}
    where each of the three terms are defined as following, 
    \begin{align}
        \mathrm{Bias}:= \left\|\big(\Xb^\top \Gb^{-1}\Xb - \Ib_d\big)\wb^{\ast}\right\|_{\Hb}^2,\quad \mathrm{Variance} = \left\|\Xb^\top \Gb^{-1}\boldsymbol{\epsilon}\right\|_{\Hb}^2,\quad \mathrm{Fluctuation}=\left\|\frac{1}{N}\sum_{j=1}^N\Delta^{(j)}\right\|_{\Hb}^2,\label{eq:_decomposition}
    \end{align}
    with the matrix $\Gb\in\RR^{n\times n}$ and the vectors $\{\Delta^{(j)}\}_{j=1}^N$ defined as following,
    \allowdisplaybreaks
    \begin{align}\label{eq:_G}
        \Gb&:=\left(\frac{\sigma^2n}{(1-\sigma^2)\eta}\cdot \Ib_n + \Ab \right)\left(\Ib_n -  \big(1-\sigma^2\big)^t\cdot \left(\Ib_n - \frac{\eta}{n}\cdot \Ab\right)^t\right)^{-1},\\
        \Delta^{(j)}&:= \sum_{k=0}^{t-1} \left(\prod_{\ell=0}^{k-1}\big(\Ib_d - \boldsymbol{\xi}_{t-\ell}^{(j)}(\boldsymbol{\xi}_{t-\ell}^{(j)})^\top\big)\big(\Ib_d - \eta \boldsymbol{\Sigma}\big)\right)
        \big(\Ib_d - \boldsymbol{\xi}_{t-k}^{(j)}(\boldsymbol{\xi}_{t-k}^{(j)})^\top\big) \cdot \frac{\eta}{n}\cdot \Xb^\top \yb \\
        &\qquad - \sum_{k=0}^{t-1} \big(1-\sigma^2\big)^{k+1}\big(\Ib_d - \eta\boldsymbol{\Sigma}\big)^k \cdot\frac{\eta}{n}\cdot\Xb^\top\yb.
    \end{align}
\end{lemma}

\begin{proof}[Proof of Lemma~\ref{lem:_average_continuous_decomposition}]
    By definition, the output $\wb_{\mathtt{avg}}$ is defined as
    \begin{align}
        \wb_{\texttt{avg}}:=\frac{1}{N}\sum_{j=1}^N \wb^{(j)}_t,\label{eq:_proof_average_continuous_decomposition_0}
    \end{align}
    where for each $j\in[N]$, the coefficient $\wb^{(j)}_t$ is given by 
    \allowdisplaybreaks
    \begin{align}
                \wb^{(j)}_t &= \sum_{k=0}^{t-1} \left(\prod_{\ell=0}^{k-1}\big(\Ib_d - \boldsymbol{\xi}_{t-\ell}^{(j)}(\boldsymbol{\xi}_{t-\ell}^{(j)})^\top\big)\big(\Ib_d - \eta \boldsymbol{\Sigma}\big)\right)
        \big(\Ib_d - \boldsymbol{\xi}_{t-k}^{(j)}(\boldsymbol{\xi}_{t-k}^{(j)})^\top\big) \cdot \frac{\eta}{n}\cdot \Xb^\top \yb\\
        & = \Delta^{(j)} + \underbrace{\sum_{k=0}^{t-1} \big(1-\sigma^2\big)^{k+1}\big(\Ib_d - \eta\boldsymbol{\Sigma}\big)^k \cdot\frac{\eta}{n}\cdot\Xb^\top\yb}_{:=\wb_t}.
    \end{align}
    Now we decompose the difference between $\wb_{\mathtt{avg}}$ in \eqref{eq:_proof_average_continuous_decomposition_0} and the truth $\wb^{\ast}$ as following, considering 
    \begin{align}
        \wb_{\texttt{avg}} - \wb^{\ast} = \frac{1}{N}\sum_{j=1}^N \wb^{(j)}_t - \wb^{\ast} = \wb_t - \wb^{\ast} + \frac{1}{N}\sum_{j=1}^N\Delta^{(j)},\label{eq:_proof_average_continuous_decomposition_1}
    \end{align}
    where the difference $\wb_t - \wb^{\ast}$ can be further explicitly expanded as 
    \begin{align}
        \wb_t - \wb^{\ast} &= \sum_{k=0}^{t-1} \big(1-\sigma^2\big)^{k+1}\big(\Ib_d - \eta\boldsymbol{\Sigma}\big)^k \cdot\frac{\eta}{n}\cdot\Xb^\top\yb - \wb^{\ast} \\
        & = \sum_{k=0}^{t-1} \big(1-\sigma^2\big)^{k+1}\big(\Ib_d - \eta\boldsymbol{\Sigma}\big)^k \cdot\frac{\eta}{n}\cdot\Xb^\top\big(\Wb \wb^{\ast} + \boldsymbol{\epsilon}\big)- \wb^{\ast} \\
        & = \big(1-\sigma^2\big)\cdot\Big(\Ib_d - \big(1-\sigma^2\big)^t\big(\Ib_d - \eta\boldsymbol{\Sigma}\big)^t\Big)\Big(\sigma^2\Ib_d + \big(1-\sigma^2\big)\eta\boldsymbol{\Sigma}\Big)^{-1}\cdot \frac{\eta}{n}\cdot \Xb^\top\Xb\wb^{\ast} - \wb^{\ast} \\
        &\qquad + \big(1-\sigma^2\big)\cdot\Big(\Ib_d - \big(1-\sigma^2\big)^t\big(\Ib_d - \eta\boldsymbol{\Sigma}\big)^t\Big)\Big(\sigma^2\Ib_d + \big(1-\sigma^2\big)\eta\boldsymbol{\Sigma}\Big)^{-1}\cdot \frac{\eta}{n}\cdot \Xb^\top\Xb\boldsymbol{\epsilon} \\
        &=\Big(\Xb^\top \Gb^{-1}\Xb - \Ib_d\Big)\wb^{\ast} + \Xb^\top\Gb^{-1}\boldsymbol{\epsilon},\label{eq:_proof_average_continuous_decomposition_2}
    \end{align}
    where the last equality uses the definition of the matrix $\Gb$ in \eqref{eq:_G} and the fact that 
    \begin{align}
        &\Big(\Ib_d - \big(1-\sigma^2\big)^t\big(\Ib_d - \eta\boldsymbol{\Sigma}\big)^t\Big)\Big(\sigma^2\Ib_d + \big(1-\sigma^2\big)\eta\boldsymbol{\Sigma}\Big)^{-1}\Xb^\top\\
        &\qquad = \Xb^\top \left(\Ib_n - \big(1-\sigma^2\big)^t\left(\Ib_d - \frac{\eta}{n} \Ab\right)^t\right)\Big(\sigma^2\Ib_n + \big(1-\sigma^2\big)\eta\Ab\Big)^{-1}.
    \end{align}
    Finally, by combining \eqref{eq:_proof_average_continuous_decomposition_1} and \eqref{eq:_proof_average_continuous_decomposition_2}, we can arrive at 
    \begin{align}
         \left\|\wb_{\mathtt{avg}} - \wb^{\ast}\right\|_{\Hb}^2   = \left\|\Big(\Xb^\top \Gb^{-1}\Xb - \Ib_d\Big)\wb^{\ast} + \Xb^\top\Gb^{-1}\boldsymbol{\epsilon} + \frac{1}{N}\sum_{j=1}^N\Delta^{(j)}  \right\|_{\Hb}^2 \leq \mathrm{Bias} + \mathrm{Variance} + \mathrm{Fluctuation}.
    \end{align}
    This completes the proof of Lemma~\ref{lem:_average_continuous_decomposition}.
\end{proof}

\begin{lemma}\label{lem:_average_continuous_g_bounds}
    The matrix $\Gb$ satisfies the that for any CoT length $t\geq \sigma^{-2}\cdot\log 2$, it holds that 
    \begin{align}
        \frac{\sigma^2n}{(1-\sigma^2)\eta}\cdot \Ib_n + \Ab \preceq  \Gb \preceq \frac{n}{\eta}\cdot\left(\frac{2}{t} + \frac{\sigma^2}{1-\sigma^2}\left(1+\frac{2}{t}\right)\right)\cdot \Ib_n + \Ab.
    \end{align}
\end{lemma}

\begin{proof}[Proof of Lemma~\ref{lem:_average_continuous_g_bounds}]
    It is direct from the definition of $\Gb$ in \eqref{eq:_G} to see the left side of the inequality. 
    To prove the right side of the inequality, consider that by \eqref{eq:_G}, we have the following,
    \begin{align}
        &\Gb - \left(\frac{\sigma^2n}{(1-\sigma^2)\eta}\cdot \Ib_n + \Ab\right) \label{eq:_proof_g_0}\\
        &\qquad = \big(1-\sigma^2\big)^t\cdot \left(\frac{\sigma^2n}{(1-\sigma^2)\eta}\cdot \Ib_n + \Ab\right) \left(\Ib_n - \frac{\eta}{n}\cdot \Ab\right)^t\left(\Ib_n -  \big(1-\sigma^2\big)^t\cdot \left(\Ib_n - \frac{\eta}{n}\cdot \Ab\right)^t\right)^{-1}.
    \end{align}
    To proceed, it suffices to consider the real-valued single-variable function $f$ defined as 
    \begin{align}
        f(x) = \frac{\big(\eta^{-1}\big(1-\sigma^2\big)^{-1}n\sigma^2 + x\big)\cdot \big(1-n^{-1}\eta x\big)^t}{1-\big(1-\sigma^2\big)^t\cdot \big(1-n^{-1}\eta x\big)^t }.
    \end{align}
    On the one hand, for $t\geq \sigma^{-2}\cdot \log 2$, we have $t>-\log2 / \log(1-\sigma^2)(1-n^{-1}\eta x)$, and thus
    \begin{align}\label{eq:_proof_g_1}
        1-\big(1-\sigma^2\big)^t\cdot \big(1-n^{-1}\eta x\big)^t  \geq \frac{1}{2}.
    \end{align}
    On the other hand, by direct calculations we can see that the numerator is upper bounded by 
    \begin{align}
        \left(\frac{\sigma^2n}{(1-\sigma^2)\eta} + x\right)\cdot \big(1-n^{-1}\eta x\big)^t \leq \frac{1}{t}\cdot \frac{n}{\eta}\cdot \left(\frac{\sigma^2}{1-\sigma^2} +1\right). \label{eq:_proof_g_1+}
    \end{align}
    Consequently, by combining \eqref{eq:_proof_g_1} and \eqref{eq:_proof_g_1+}, we can see that for $t\geq \sigma^{-2}\cdot \log 2$,  
    \begin{align}
        f(x) \leq \frac{2}{t}\cdot \frac{n}{\eta}\cdot \left(\frac{\sigma^2}{1-\sigma^2} +1\right),
    \end{align}
    which, combined with \eqref{eq:_proof_g_0}, further indicates that 
    \begin{align}
        \Gb - \left(\frac{\sigma^2n}{(1-\sigma^2)\eta}\cdot \Ib_n + \Ab\right) \preceq \frac{2}{t}\cdot \frac{n}{\eta}\cdot \left(\frac{\sigma^2}{1-\sigma^2} +1\right)\cdot \Ab.
    \end{align}
    This completes the proof of the right side inequality of Lemma~\ref{lem:_average_continuous_g_bounds} and finishes the proof.
\end{proof}

\begin{lemma}[Bias error]\label{lem:_aggregate_continuous_bias}
    Under Assumption~\ref{ass:_data_distribution}, taking the step size $\eta \lesssim \mathrm{Tr}(\Hb)^{-1}$ and for any $k\in[d]$, with probability at least $1-1/\mathrm{poly}(n)$, it holds that 
    \begin{align}
        \mathbb{E}_{\wb^{\ast}}[\mathrm{Bias}]\lesssim \omega^2\cdot \left(\frac{1}{n^2}\cdot\left(\frac{n}{\eta}\cdot\left(\frac{2}{t} + \frac{\sigma^2}{1-\sigma^2}\cdot\left(1+\frac{2}{t}\right)\right) + \sum_{k<i\leq d}\lambda_i\right)^2\cdot\sum_{1\leq i\leq k}\frac{1}{\lambda_i} + \sum_{k<i\leq d} \lambda_i\right).
    \end{align}
\end{lemma}

\begin{proof}[Proof of Lemma~\ref{lem:_aggregate_continuous_bias}]
    According to the definition of $\mathrm{Bias}$ in \eqref{eq:_decomposition}, using that $\wb^{\ast}\sim \cN(\boldsymbol{0},\omega^2\cdot\Ib_d)$ we have 
    \allowdisplaybreaks
    \begin{align}
        \mathbb{E}_{\wb^{\ast}}[\mathrm{Bias}]&= \mathbb{E}_{\wb^{\ast}\sim \cN(\boldsymbol{0},\omega^2\cdot\Ib_d)}\left[\left\|\Hb^{\frac{1}{2}}\big(\Ib_d - \Xb^\top \Gb^{-1}\Xb\big)\wb^{\ast}\right\|_2^2\right]\\
        &= \omega^2\cdot \mathrm{Tr}\left(\Hb\big(\Ib_d - \Xb^\top \Gb^{-1}\Xb\big)^2\right)\\
        &\leq \omega^2\cdot  \mathrm{Tr}\left(\Hb\left(\Ib_d - \Xb^\top \left(\frac{n}{\eta}\cdot\left(\frac{2}{t} + \frac{\sigma^2}{1-\sigma^2}\left(1+\frac{2}{t}\right)\right)\cdot \Ib_n + \Ab\right)^{-1}\Xb \right)^2\right),
    \end{align}
    where the last inequality follows from Lemma~\ref{lem:_average_continuous_g_bounds}.
    Notice that the quantity of trace on the right hand side actually corresponds to the bias error of the standard ridge regression with regularization coefficient $\widetilde{\lambda}_{\mathrm{effect}}$ of 
    \begin{align}
        \widetilde{\lambda}_{\mathrm{effect}}^{\mathrm{Bias}}:= \frac{n}{\eta}\cdot\left(\frac{2}{t} + \frac{\sigma^2}{1-\sigma^2}\left(1+\frac{2}{t}\right)\right).
    \end{align}
    Thus by invoking Theorem 1 of \cite{tsigler2023benign}, we can then obtain the result in Lemma~\ref{lem:_aggregate_continuous_bias}.
\end{proof}

\begin{lemma}[Variance error]\label{lem:_aggregate_continuous_variance}
    Under Assumption~\ref{ass:_data_distribution}, taking the step size $\eta\lesssim \mathrm{Tr}(\Hb)^{-1}$ and for any $k\in[d]$, with probability at least $1-1/\mathrm{poly}(n)$, it holds that 
    \begin{align}
        \mathbb{E}_{\boldsymbol{\epsilon}}\left[\mathrm{Variance}\right] \lesssim \sigma_{\epsilon}^2\cdot\left(\frac{k}{n} + n\cdot \left(\frac{\sigma^2n}{(1-\sigma^2)\eta} + \sum_{k<i\leq d}\lambda_i\right)^{-2}\cdot \sum_{k<i\leq d}\lambda_i^2\right).
    \end{align}
\end{lemma}

\begin{proof}[Proof of Lemma~\ref{lem:_aggregate_continuous_variance}]
    According to the definition of $\mathrm{Bias}$ in \eqref{eq:_decomposition}, using that $\epsilon_i\sim \cN(0,\sigma_{\epsilon}^2)$ we have 
    \begin{align}
        \mathbb{E}_{\boldsymbol{\epsilon}}\left[\mathrm{Variance}\right] &= \mathbb{E}_{\epsilon\sim \cN(\boldsymbol{0}, \sigma_\epsilon^2\cdot \Ib_d)}\left[\left\|\Hb^{\frac{1}{2}}\Xb^\top \Gb^{-1}\boldsymbol{\epsilon}\right\|_{2}^2\right]\\
        &= \sigma_{\epsilon}^2\cdot \mathrm{Tr}\left(\Xb\Hb\Xb^\top\Gb^{-2}\right) \\
        &\leq \sigma_{\epsilon}^2\cdot \mathrm{Tr}\left(\Xb\Hb\Xb^\top\left(\frac{\sigma^2n}{(1-\sigma^2)\eta}\cdot \Ib_n + \Ab \right)^{-2}\right)
    \end{align}
    Similar to the proof of Lemma~\ref{lem:_aggregate_continuous_bias}, the above quantity on the right hand side actually corresponds to the variance error of standard ridge regression with regularization coefficient $\widetilde{\lambda}_{\mathrm{effect}}$ of 
    \begin{align}
        \widetilde{\lambda}_{\mathrm{effect}}^{\mathrm{Var}}:=\frac{\sigma^2n}{(1-\sigma^2)\eta}.
    \end{align}
    Consequently, by Theorem 1 of \cite{tsigler2023benign}, we can obtain the result in Lemma~\ref{lem:_aggregate_continuous_variance}.
\end{proof}

\begin{lemma}[Fluctuation error]\label{lem:_aggregate_continuous_fluctuation}
Suppose that we choose $\sigma^2 < 1 /(d+1)$ and the step size $\eta\lesssim \mathrm{Tr}(\Hb)^{-1}$. Then there exists an event with probability $1-1/\mathrm{poly}(n)$ over the randomness of $\bX$ on which it holds that 
\begin{align}
    \EE_{\wb^*, \bxi,\boldsymbol \epsilon}[\mathrm{Fluctuation}] \lesssim \frac{( \eta  \sigma^{-2}\sigma_\epsilon^2 d\cdot \Tr(\Hb) /n + \omega^2)\cdot \norm{\Hb}_2 }{N}
\end{align}
\end{lemma}

\begin{proof}[Proof of Lemma~\ref{lem:_aggregate_continuous_fluctuation}]

In the proof, we replace the notation $\Delta^{(j)}$ with $\Delta^j_{t}$ to emphasize the dependence on the reasoning step. 
From the characterization in Lemma~\ref{lem:_average_continuous_decomposition}, we have for each path and its expectation over $\bxi$, it holds that
\begin{align}
    \wb_{t+1}^{(j)} &=  (\Ib - \boldsymbol{\xi}^{(j)}_{t+1}{\boldsymbol{\xi}_{t+1}^{(j)}}^{\top }) (\Ib - \eta \bSigma) \big(\wb_t^{(j )}+  \eta \Xb^\top \yb /n \big) \\ 
    &= (1- \sigma^2)\cdot (\Ib - \eta \bSigma) (\wb_t^{(j)} + \eta \Xb^\top \yb / n) + \sigma^2\cdot \Big(\Ib - \sigma^{-2} \boldsymbol{\xi}^{(j)}_{t+1}{\boldsymbol{\xi}_{t+1}^{(j)}}^{\top }\Big) ( \wb_t^{(j)} + \eta \Xb^\top \yb / n)\\
    \wb_{t+1} &= (1-\sigma^2) ( \Ib - \eta \bSigma) (\wb_t + \eta\Xb^\top \yb / n). 
\label{eq:average_path_iteration} 
\end{align}
Since there exists an event with probability $1-1/ \mathrm{poly}(n)$ on which $\Tr(\bSigma)\gtrsim \Tr(\Hb)$, we have that $\eta<1/\Tr(\bSigma)$ with high probability. 
In order to control the fluctuation error, we begin with deriving a deterministic upper bound on $\wb_t$. 
\paragraph{Bounding the expected path.} 
By \eqref{eq:average_path_iteration}, the quantity $\gb_t = \wb_t + \eta \Xb^\top \yb$ can be iteratively characterized as follows:
\begin{align}
    \gb_{t+1} &= (1-\sigma^2) ( \Ib - \eta \bSigma) \gb_t + \eta \Xb^\top \yb/ n\\
    &= \sum_{k=0}^{t} (1-\sigma^2)^k \big(\Ib - \eta \bSigma\big)^k \eta \Xb^\top \yb /n\\
    &= \sum_{k=0}^t \big(\Ib - \sigma^2 \Ib -\eta\bSigma +\eta\sigma^2\bSigma)^k \eta \bSigma\wb^* \\ &\qquad +  \sum_{k=0}^t \big(\Ib - \sigma^2 \Ib -\eta\bSigma +\eta\sigma^2\bSigma)^k \eta \Xb^\top \boldsymbol\epsilon / n,
\end{align} 
To this end, we define $p(z) = \sum_{k=0}^t(1-\sigma^2 - z + \sigma^2z)^
{k}$.
We can bound the scalar polynomials $p(z)$, $p(z)\cdot z$ and $p^2(z)\cdot z$ on $[0, 1)$ as 
\begin{align}
    p(z) &\le \frac{1}{\sigma^2 + (1-\sigma^2 )z };\\
    p(z) \cdot  z&\le \frac{z}{\sigma^2 + (1-\sigma^2 )z} \lesssim (\sigma^{-2}z )\wedge 1; \label{eq:pz_z_bound}\\
    p^2(z)\cdot z &\le \frac{z}{\big(\sigma^2  + (1-\sigma^2 ) z\big)^2}  \lesssim (\sigma^{-4} z )\wedge z^{-1}.\label{eq:z_pz_z_bound}
\end{align}
We begin with the first term. 
It follows from \eqref{eq:pz_z_bound} that 
\begin{align}
    \norm{p(\eta\bSigma) \cdot \eta\bSigma}_2  &\lesssim (\sigma^{-2} \cdot \eta \norm{\bSigma}_2)\wedge 1. 
\end{align}
Therefore the first term can be upper bounded by $ \big((\sigma^{-2} \eta \norm{\bSigma}_2 ) \wedge 1\big)\cdot \norm{\wb^*}_2$.  
For the second term, we have that 
\begin{align}
    \EE_{\boldsymbol{\epsilon}}\Big[\Big\|\sum_{k=0}^t \big(\Ib - \sigma^2 \Ib -\eta\bSigma +\eta\sigma^2\bSigma)^k \eta \Xb^\top \boldsymbol\epsilon/n\Big\|_2^2\Big] &= \frac {\eta  \sigma_\epsilon^2}{n} \cdot \Tr\Big( p(\eta \bSigma)\cdot \eta\bSigma\cdot p(\eta \bSigma)\Big),
\end{align}
And therefore we have by \eqref{eq:z_pz_z_bound} that 
\begin{align}
\EE_{\boldsymbol\epsilon,\wb^*}[\sup_{t\ge 0 }\norm{\gb_{t}}_2^2] & \lesssim  \frac{ \eta\sigma_\epsilon^2 }{n}\cdot \sigma^{-4} \Tr(\bSigma)+  \big(1  \wedge  \sigma^{-2 }\eta\norm{\bSigma}_2\big) \norm{\wb^*}_2^2 \\ 
&\lesssim    \frac{ \eta\sigma_\epsilon^2 }{n}\sigma^{-4} \Tr(\bSigma)  +  \norm{\wb^\star}_2^2.
\end{align}

\paragraph{Bounding the fluctuation.} In the following, we use $\bLambda_{t}^{(j)} = (\Ib - \sigma^{-2}\boldsymbol{\xi}^{(j)}_{t+1}{\boldsymbol{\xi}_{t+1}^{(j)}}^{\top })$ for abbreviation. 
The fluctuation term $\Delta_t^{(j)}$ follows that 
\begin{align}
    \Delta_{t+1}^{(j)} &=  \wb_{t+1}^{(j)}- \wb_{t+1} \\ 
    &= (1-\sigma^2) \cdot (\Ib - \eta\bSigma)\cdot\Delta_t^{(j)} + \sigma^2\cdot \bLambda_t^{(j)} \cdot (\wb_t^{(j)}  +\eta \Xb^\top \yb). \label{eq:fluc_iteration}
\end{align} 
For each $t$, we have that $\bLambda^{(j)}_t$ is independent with $\wb_t^{(j)}$ and is of zero mean. 
Consequently we have that $\EE[\Delta_t^{(j)}] = 0$ for any $t\ge 0$. 
Besides, it can be easily verified by induction that $\Delta_t^{(j)},{j\le N}$ are independent and identically distributed. 
Thanks to this, we have that
\begin{align}
    \EE\Big[\Big\|N^{-1}\sum_{j\le N }\Delta_t^{(j)}\Big\|^2_{\Hb}\Big] &= \EE\Big[N^{-2}\sum_{j \le N} {\Delta_t^{(j)}} ^\top \Hb\Delta_t^{(j)}   + N^{-2} \sum_{j<k} {\Delta_t^{(j)} }^\top  \Hb \Delta_t^{(k)}  \Big] \\ 
    &= N^{-1} \dotp{\Hb}{\EE[{\Delta_t^{(j)}}^\top {\Delta_t^{(j)}}  ]}. \label{eq:fluc_error_bound_0}
\end{align} 
Therefore, it suffices to upper bound the second moment of the fluctuation along a single reasoning path. 
For simplicity, let us drop the superscript $(j)$ in the subsequent analysis. 
We study the iteration of the second moment $\Sbb_t =  \EE[\Delta_t\Delta_t^\top]$.  
Rewriting \eqref{eq:fluc_iteration}, we get that
\begin{align}
    \Delta_{t+1} &= (1-\sigma^2) \cdot (\Ib -\eta \bSigma) \Delta_t + \sigma^2 \bLambda_t  \Delta_t \\ 
    &\qquad + \sigma^2 \bLambda_t  \cdot(\wb_t  +\eta\Xb^\top \yb ).
\end{align}
Note that $\bLambda_t$ and $\Delta_t$ are zero mean and independent, we have that
\begin{align}
    \Sbb_{t+1} &=  (1-\sigma^2) ^2\cdot (\Ib - \eta \bSigma) \Sbb_t (\Ib - \eta \bSigma)  \\ 
    & \qquad + \sigma^4 \cdot \EE[\bLambda_t \Delta_t\Delta_t^\top \bLambda_t^\top ]  +  \sigma^4 \eta^2  \cdot  \EE[\bLambda_t \wb_t \wb_t^\top \bLambda_t^\top] \\ 
    & = (1-\sigma^2)^2\cdot  (\Ib - \eta \bSigma) \Sbb_t (\Ib - \eta \bSigma) \\
    &\qquad + \sigma^4 \big(\tr(\Sbb_t) \Ib + \diag(\Sbb_t)\big) + \sigma^4  \cdot \big(   \tr(\gb_t \gb_t^\top) \Ib +    \diag( \gb_t \gb _t^\top )\big). 
\end{align}
Here the second identity follows from Lemma~\ref{lem:lambda_A_lambda_exp} and $\gb_t = \wb_t +\eta \Xb^\top \yb$.  
The structure of this iteration has two folds. 
The first part is that the gradient step, together with the average effect of the noise term, help to decay the second moment of the fluctuation. 
The second part is that the noise term re-allocate the fluctuation in the last step to the current step in an isotropic manner. 
Since $\tr(\bA)$ prevails over $\diag(\bA)$, we can continue as 
\begin{align}
    \Tr(\Sbb_{t+1}) &\le (1-\sigma^2)^2\cdot  \norm{\Ib -\eta \bSigma}_{{2}}^2 \Tr( \Sbb_t) + \sigma^4(d+1)\cdot  \big(\tr(\Sbb_t)  + \tr(\gb_t \gb_t^\top) \big) \\
    &\le \Big((1-\sigma^2)^2\cdot  \norm{\Ib -\eta \bSigma}_{{2}}^2 + \sigma^4 (d+1)\Big)\cdot \Tr( \Sbb_t)  + \sigma^4 (d+1)\max_{t\ge 0} \norm{\gb_t}_2^2. \label{eq:trace_itr_upper}
\end{align} 
Based on our assumption that $\sigma^2 <(d+1)^{-1}$, it holds by the convexity of the quadratic function that
\begin{align}
    (1-\sigma^2)^2\cdot  \norm{\Ib -\eta \bSigma}_{{2}}^2 + \sigma^4 (d+1)&\le (1-\sigma^2)^2 + \sigma^4 (d+1) \\
    &\le 1-\frac{d\sigma^2}{d+1}. 
\end{align}
Plugging this back to \eqref{eq:trace_itr_upper}, we have that
\begin{align}
    \Tr(\Sbb_{t+1})  &\le  \frac{\sigma^4 \cdot(d+1)\cdot \max_{t\ge 0} \norm{\gb_t}_2^2}{ 1 - (1-\sigma^2)^2\cdot  \norm{\Ib -\eta \bSigma}_{{2}}^2 - \sigma^4 (d+1)} \\ 
    &\le \frac{(d+1)^2\sigma^2}{ d}\cdot \max_{t\ge 0} \norm{\gb_t}_2^2. 
\end{align}  
Now we can leverage \eqref{eq:fluc_error_bound_0} and get that 
\begin{align}
\EE_{\boldsymbol{\epsilon}, \wb^\star, \bxi}\Big[\Big\|\frac{1}{N} \sum_{j=1}^N \Delta^{(j)}\Big\|_{\Hb}^2\Big] &\le \EE_{\boldsymbol{\epsilon}, \wb^\star}\big[N^{-1} \Tr(\Sbb_{t}) \cdot \norm{\Hb}_2 \big] \\
&\lesssim  \frac{(d+1)^2\sigma^2}{  Nd } \cdot\Big(\frac{\eta\sigma^2_\epsilon}{n}\cdot\sigma^{-4} \Tr(\bSigma) + \EE_{\wb^*}[\norm{\wb^*}_2^2  ]\Big) \cdot \norm{\Hb}_2 \\
&\lesssim  \frac{( \eta  \sigma^{-2 }\sigma_\epsilon^2 d\cdot \Tr(\Hb) /n + \omega^2)\cdot \norm{\Hb}_2 }{N}.
\end{align}
The last inequality use that $\Tr(\bSigma)\lesssim \Tr(\Hb)$ with high probability. 
This concludes the proof for the fluctuation error. 
\end{proof}



Now with the above lemmas, we are ready to conclude and prove Theorem~\ref{thm:_average_continuous} for Example~\ref{exp:_linear}. 

\begin{proof}[Proof of Theorem~\ref{thm:_average_continuous} for Example~\ref{exp:_linear}]  
    Combining Lemma~\ref{lem:_average_continuous_decomposition}, Lemma~\ref{lem:_aggregate_continuous_bias}, Lemma~\ref{lem:_aggregate_continuous_variance}, and Lemma~\ref{lem:_aggregate_continuous_fluctuation} gives the desired result.
\end{proof}

\subsection{Proof of Theorem~\ref{cor:_average_continuous}}\label{subsec:_proof_cor_average_continuous}

\subsubsection{Proof for Example~\ref{exp:_constant}}

\begin{proof}[Proof of Theorem~\ref{cor:_average_continuous} for Example~\ref{exp:_constant}]
    This follows directly from Theorem~\ref{thm:_average_continuous} for Example~\ref{exp:_constant} and the proof of Proposition~\ref{cor:_gd_continuous_polynomial_decay}.
\end{proof}

\subsubsection{Proof for Example~\ref{exp:_linear}}

\begin{proof}[Proof of Theorem~\ref{cor:_average_continuous} for Example~\ref{exp:_linear}]
    This follows from Theorem~\ref{thm:_average_continuous} for Example~\ref{exp:_linear}, and repeating the proof of Proposition~\ref{cor:_average_continuous} for $k^{\ast}_{\mathrm{Bias}}$ and $k^{\ast}_{\mathrm{Var}}$ in Theorem~\ref{thm:_average_continuous}.
\end{proof}

\subsection{Technical Results}
\begin{lemma} \label{lem:lambda_A_lambda_exp}
For any deterministic matrix $\Ab\in \RR^{d\times d}$ and $\boldsymbol{\xi}\sim \cN(\boldsymbol 0_d, \Ib_d)$, it holds that
\begin{align}
    \EE[(\Ib - \boldsymbol{\xi}\boldsymbol{\xi}^\top ) \Ab(\Ib - \boldsymbol{\xi}\boldsymbol{\xi}^\top )] &= \tr(\Ab) \Ib_d + \diag(\Ab), 
\end{align}
where $(\diag(\Ab))_{ij} = \delta_{ij}\cdot A_{ij}$ and $\delta_{ij}$ is the Kronecker delta.
\end{lemma}

\begin{proof}[Proof of Lemma~\ref{lem:lambda_A_lambda_exp}]
Note that the $(i,j)$-entry of $\Ib  - \boldsymbol{\xi}\boldsymbol{\xi}^\top$ is $\delta_{ij} - \xi_i \xi_j$. 
First of all, it is clear that whenever $|\{i,j\}\setminus \{k,l\} | \ge 1$ or $|\{k,l\}\setminus \{i,j\} | \le 1 $, we have that $\EE[(\delta_{ij} - \xi_i \xi_j)\cdot (\delta_{kl} - \xi_k \xi_l)] = 0$. 
So the only non-trivial cases are that: (i) $i=j=k=l$; (ii) $\{ i,j\} = \{k,l\}$ and  $i\neq j$. For the first case, we have that $\EE[(\delta_{ij} - \xi_i \xi_j)\cdot (\delta_{kl} - \xi_k \xi_l)] = \EE[(\xi_i \xi_j)^2 ]=1$. 
For the second case, we have that $\EE[(1 - \xi_i^2)^2]=\EE[\xi_i^4] - \EE[\xi^2_i]^2 = 2$. 

Given this we have for $i\neq j$ that 
\begin{align}
    \EE[\bLambda \Ab \bLambda]_{i,j} &= \EE[\sum_{k,l=1}^d \bLambda_{ ik} \Ab_{kl} \bLambda_{lj}]=0,
\end{align}
because each summand is zero since $i\neq j$. 
For the diagonal terms, we have that 
\begin{align}
    \EE[\bLambda \Ab \bLambda]_{i,i} &= \EE[\sum_{k,l=1}^d \bLambda_{ ik} \Ab_{kl} \bLambda_{li }] \\ 
    &=  \EE[\sum_{k=1}^d \bLambda_{ ik} \Ab_{kk} \bLambda_{ki }] \\ 
    &=  \EE[\sum_{k\neq i} \bLambda_{ ik} \Ab_{kk} \bLambda_{ki }] +\EE[ \bLambda_{ ii} \Ab_{ii} \bLambda_{ii }]\\ 
    &=  \tr(\Ab) + \Ab_{ii}.
\end{align}
Thus the desired result follows. 
\end{proof}

%% file: icml2025/appendix/binary.tex
\section{Proofs for Section \ref{sec:binary_analysis}}
\paragraph{Notation} We let $[n]$ denote the set of indices from $1$ to $n$. Boldface uppercase letters such as $\Xb$ represent matrices, while boldface lowercase letters such as $\xb$ denote vectors. Specifically, $\xb[i]$ denotes the $i$-th element of $\xb$.
\subsection{Proof of Theorem \ref{th:sc_maj_vote}}\label{subec:_proof_th_sc_maj_vote}
\begin{proof}[Proof of Proposition~\ref{th:sc_maj_vote}]
Considering that we sample $N$ different $\wb_t$ from the distribution $\{p(\wb_t = \wb)\}_{\wb \in \cW}$ to obtain $\Wb = \{\wb^{(1)}_t, \dots, \wb^{(N)}_t\}$. Let $\mathtt{Count}(\wb)$ represent the frequency of occurrence of $\wb$ in $\Wb$. 
For each $\wb' \in \cW\setminus \{\wb^{\ast}\}$, we upper bound the probability of $\mathtt{Count}(\wb') > \mathtt{Count}(\wb^*)$. 
To this end, we define $N$ random variables $a_1, \cdots, a_N$ such that $a_i = 1$ if $\wb^{(i)}_t = \wb^*$, $a_i = -1$ if $\wb^{(i)}_t = \wb'$, and $a_i = 0$ otherwise. This leads to the following bound,
\begin{equation*} \begin{aligned} 
\PP(\mathtt{Count}(\wb') > \mathtt{Count}(\wb^*)\mid \wb_0,\cD) \leq \PP\left(\sum_{i = 1}^N a_i \leq 0\mid \wb_0,\cD\right) \leq \exp\left(-\left(p(\wb_t = \wb^*) - p(\wb_t = \wb')\right)^2 \cdot \frac{N}{2}\right),
\end{aligned} \end{equation*} 
where the last inequality is due to Hoeffding’s inequality. Then
\begin{equation*} \begin{aligned} 
\sum_{\wb' \in\cW\setminus\{\wb^{\ast}\}}  \PP(\wb^{\mathtt{mv}}_{t,N} = \wb' \mid \wb_0,\cD)
&\leq \sum_{\wb' \in\cW\setminus\{\wb^{\ast}\}}  \PP(\mathtt{Count}(\wb') > \mathtt{Count}(\wb^*)\mid \wb_0,\cD) \\
&\leq \sum_{\wb' \in\cW\setminus\{\wb^{\ast}\}} \exp\left(- \frac{N}{2}\cdot \big(p(\wb^*) - p(\wb')\big)^2\right) \\
&\leq |\cW\setminus\{\wb^{\ast}\}| \cdot \exp\left(-\frac{N}{2}\cdot\Delta_t^2 \right),
\end{aligned} \end{equation*} 
where the final inequality is based on the definition of $\Delta_t = p(\wb^*) - \max\limits_{\wb' \in \cW\setminus\{\wb^{\ast}\}}p(\wb')$.
Consequently, 
\begin{equation*} \begin{aligned} 
\PP(\wb^{\mathtt{mv}}_{t,N} = \wb^*\mid \wb_0,\cD) \geq 1 - \sum_{\wb' \in\cW\setminus\{\wb^{\ast}\}}  \PP(\wb^{\mathtt{mv}}_{t,N} =  \wb'\mid \wb_0,\cD) \geq 1 - |\cW|\cdot \exp\left(-\frac{N}{2}\cdot\Delta_t^2\right).
\end{aligned} \end{equation*} 
This completes the proof of Proposition~\ref{th:sc_maj_vote}.
\end{proof}

\subsection{Proof of Theorem \ref{th:sc_sufficient_n}}\label{subsec:_proof_th_sc_sufficient_n}
Here, we first establish bounds for each element in $\tilde{\wb}_t$ in \autoref{lem:bound_wb_t}. Next, in \autoref{lem:greedy_search}, we prove $\wb_T$ will converge to $\wb^*$ for both greedy decoding and majority vote algorithm. Lastly, in \autoref{lem:maj_convergence}, we demonstrate the convergence rate for greedy decoding as shown in \autoref{th:sc_sufficient_n}.

\begin{lemma}\label{lem:bound_wb_t}
Given $\tilde{\wb}_t = \wb_{t-1} - \frac{1}{n}\left(\Xb \Xb^\top \wb_{t-1} - \Xb \Yb^\top\right)$, where $\Yb = \wb^* \Xb + \mathbf{\epsilon}$, We  define $\cE_1$ as follows:
\begin{equation*} \begin{aligned} 
\cE_1 := \left\{
    \begin{aligned}&\wb^*[i] + \frac{2k + \sigma_\epsilon}{n^{1/4}} \geq \tilde{\wb}_t[i] \geq \wb^*[i] - \frac{2k + \sigma_\epsilon}{n^{1/4}}, \\
&\text{specifically  when } \wb_{t-1} = \wb^*, 
\wb^*[i] + \frac{\sigma_\epsilon}{n^{1/4}} \geq \tilde{\wb}_t[i] \geq \wb^*[i] - \frac{\sigma_\epsilon}{n^{1/4}} \end{aligned}
    \right\}, 
\end{aligned} \end{equation*} 
then $\cE_1$ holds with probability at least $1 - \delta$, where $\delta = 2\left(d^2+ 2d\right)e^{-c n^{1/2}}$.
\end{lemma}
\begin{proof}
\begin{equation} \begin{aligned} \label{eq:bw_expand} 
\tilde{\wb}_t [i] &= \wb_{t-1} [i] - \frac{1}{n}\sum_{j \in [n], l \in [d]}\left(x_{ji}x_{jl}\wb_{t-1} [l] - x_{ji}x_{jl}\wb^* [l]\right) + \frac{1}{n}\sum_{j \in [n]} x_{ji} \epsilon_i \\
&= \wb_{t-1} [i] - \frac{1}{n} \left(\wb_{t-1} [i] - \wb^* [i]\right) \underbrace{\sum_{j \in [n]}x_{ji}^2}_{A_i}  - \frac{1}{n}\sum_{l \in [d], l \neq i}\left(\wb_{t-1} [l] - \wb^* [l]\right)\underbrace{\sum_{j \in [n]}\left(x_{ji}x_{jl}\right)}_{B_{il}} + \frac{1}{n}\sum_{j \in [n]} x_{ji} \epsilon_i \\
&= \wb_{t-1} [i] - \frac{1}{n} \left(\wb_{t-1} [i] - \wb^* [i]\right) A_i  - \frac{1}{n}\sum_{l \in [d], l \neq i}\left(\wb_{t-1} [l] - \wb^* [l]\right)B_{il} + \frac{1}{n}\sum_{j \in [n]} x_{ji} \epsilon_i.
\end{aligned} \end{equation} 

Since $x_{i j} \sim \cN\left(0, 1\right)$ for any $i,j$, by Lemma 2.7.7 and Bernstein's inequality in \cite{vershynin2018high}, there exists an absolute constant $c_1$ such that

\begin{equation*} \begin{aligned} 
\PP\{|\sum_i x_{ji} x_{jl}| \leq t\} \leq 2 \exp{\left(-c_1 \min\left(
\frac{t^2}{\sum_j || x_{ji} x_{jl}||^2_{\psi_i}},
\frac{t}{\max_j || x_{ji} x_{jl}||_{\psi_i}}
\right)\right)},
\end{aligned} \end{equation*} 
where $||.||_{\psi_1}$ denotes to the sub-exponential norm. Besides, $|| x_{ji} x_{jl}||_{\psi_i} \leq ||x_{ji}||_{\psi_2} \cdot ||x_{j  k}||_{\psi_2} \leq C_1^2$, with the last inequality derived from the properties of the Gaussian distribution, where $C_1$ is a constant. Furthermore, we have:
\begin{equation} \begin{aligned}\label{eq:bern_concen1}
\PP\{|B_{il}| \leq t_1\} \leq 2 \exp{\left(-c_1 \min\left(
\frac{t_1^2}{n C_1^4},
\frac{t_1}{C_1^2}.
\right)\right)}
\end{aligned} \end{equation} 
Similarly we have
\begin{equation} \begin{aligned} \label{eq:bern_concen2}
\PP\{|\sum_{j \in [n]} x_{ji} \epsilon_i| \leq t_2\} \leq 2 \exp{\left(-c_2 \min\left(
\frac{t_2^2}{n C_1^4 \sigma_\epsilon^2},
\frac{t_2}{C_1^2 \sigma_\epsilon}.
\right)\right)}
\end{aligned} \end{equation}

For $A_i = \sum_{j \in [n]}x_{ji}^2$, since $x_{ji}^2 - 1$ are sub-exponential and mean zero random variables, we can directly apply Bernstein's inequality to obtain: 
\begin{equation} \begin{aligned} \label{eq:bern_concen3}
\PP\{|A_i - n| \leq t_3\} \leq 2 \exp{\left(-c_3 \min\left(
\frac{t_3^2}{n C_3^4},
\frac{t_3}{C_3^2}
\right)\right)}
\end{aligned} \end{equation}

By setting $t_1 = t_3 = n^{3/4}, t_2 = \sigma_\epsilon n^{3/4}$, $c = \frac{\min\left(c_1,c_2,c_3\right)}{\max\left(C_1^4,C_2^4,C_3^4,C_1^2,C_2^2,C_3^2\right)}$, and applying the derived \autoref{eq:bern_concen1}, \autoref{eq:bern_concen2}, \autoref{eq:bern_concen3} for all $i,l \in [d]$, we establish that
\begin{equation} \begin{aligned}\label{eq:concen_w_all}
&|B_{il}| \leq n^{3/4} & \forall i,l \in [d]; \\
&|\sum_{j \in [n]} x_{ji} \epsilon_i| \leq \sigma_\epsilon n^{3/4} & \forall i \in [d]; \\
&|A_i - n| \leq n^{3/4} & \forall i \in [d],
\end{aligned} \end{equation}
holds with a probability of at least $1 - 2\left(d^2+ 2d\right)e^{-c n^{1/2}}$. Hereafter, we condition on \autoref{eq:concen_w_all}. 

By combining  \autoref{eq:concen_w_all} with \autoref{eq:bw_expand}, the following equation is obtained: 
\begin{equation*} \begin{aligned}
\tilde{\wb}_t [i] &=\wb_{t-1} [i] - \frac{1}{n} \left(\wb_{t-1} [i] - \wb^* [i]\right) A_i  - \frac{1}{n}\sum_{l \in [d], l \neq i}\left(\wb_{t-1} [l] - \wb^* [l]\right)B_{il} + \frac{1}{n}\sum_{j \in [n]} x_{ji} \epsilon_i \\
&\leq \wb^*[i] + \frac{1}{n^{1/4}}\sum_{l \in [d]}|\wb_{t-1} [l] - \wb^* [l]| + \frac{\sigma_\epsilon}{n^{1/4}} \\
&\leq \wb^*[i] + \frac{2k + \sigma_\epsilon}{n^{1/4}},
\end{aligned} \end{equation*} 
the final inequality is by $||\wb_{t}||_0 = k \left(t \geq 1\right)$ and $||\wb_{0}||_0 = 0$. Similarly we have.
\begin{equation*} \begin{aligned}
\tilde{\wb}_t [i] 
&\geq \wb^*[i] - \frac{1}{n^{1/4}}\sum_{l \in [d]}|\wb_{t-1} [l] - \wb^* [l]| - \frac{\sigma_\epsilon}{n^{1/4}} \\
&\geq \wb^*[i] - \frac{2k + \sigma_\epsilon}{n^{1/4}}
\end{aligned} \end{equation*}
Specifically, when $\wb_{t-1} = \wb^*$,
\begin{equation*} \begin{aligned}
\wb^*[i] + \frac{\sigma_\epsilon}{n^{1/4}} \geq \tilde{\wb}_t [i] &\geq \wb^*[i] - \frac{\sigma_\epsilon}{n^{1/4}}.
\end{aligned} \end{equation*}
\end{proof}

Without loss of generality, in the following we assume the first $k$ elements of $\wb^*$ are 1, and others are 0. We define $\cC^{\left(m\right)}$ as the set of all possible permutations for $[m]$. 
\begin{lemma}[Perfect Accuracy for Both Greedy Decoding and Majority Vote]\label{lem:greedy_search}
Given $\tilde{\wb}_t = \wb_{t-1} - \frac{1}{n}\left(\Xb \Xb^\top \wb_{t-1} - \Xb \Yb^\top\right)$, where $\Yb = \wb^* \Xb + \mathbf{\epsilon}$, suppose $\cE_1$ holds, $\frac{2k + \sigma_\epsilon}{n^{1/4}} < \frac13$ and sampling number $N$ is sufficient large, then for all $t \geq 1$, we have
\begin{equation*}
\begin{aligned}
\wb_t^{\mathtt{maj}\cdot N}  = \wb_t^{\mathtt{greedy}} = \wb^*.
\end{aligned}
\end{equation*}
\end{lemma}

\begin{proof}
Given that $\cE_1$ holds,  for $t \geq 1$:
\begin{equation*} \begin{aligned}
 \begin{cases}
\tilde{\wb}_{t}[i] \geq 1 - \frac{2k + \sigma_\epsilon}{n^{1/4}} > 1/2 & i \leq k \\
\tilde{\wb}_{t}[i] \leq \frac{2k + \sigma_\epsilon}{n^{1/4}} < 1/2 & k < i \leq d 
\end{cases}.
\end{aligned} \end{equation*}

In this case we observe that $\tilde{\wb}_{t}[i] > \tilde{\wb}_{t}[j]$ for all $i \leq k$ and $k < i \leq d$. Without loss of generality, we further assume 
\begin{equation*} \begin{aligned}
\tilde{\wb}_{t}[1] \geq \tilde{\wb}_{t}[2] \geq \cdots \geq \tilde{\wb}_{t}[k] >  \tilde{\wb}_{t}[k+1] \geq \tilde{\wb}_{t}[k+2] \geq \cdots \geq \tilde{\wb}_{t}[d].
\end{aligned} \end{equation*}
For $p_{\tilde{\wb}_{t}}[i] = \frac{\max\left(0,\tilde{\wb}_{t}\right)}{\sum_{j = 1}^d \max\left(0,\tilde{\wb}_{t}\right)}$, we also have
\begin{equation*} \begin{aligned}
p_{\tilde{\wb}_{t}}[1] \geq p_{\tilde{\wb}_{t}}[2] \geq \cdots \geq p_{\tilde{\wb}_{t}}[k] >  p_{\tilde{\wb}_{t}}[k+1] \geq p_{\tilde{\wb}_{t}}[k+2] \geq \cdots \geq p_{\tilde{\wb}_{t}}[d].
\end{aligned} \end{equation*}

Then for $\wb' \in \cW_{/\wb^*}$ where the index of nonzero elements are $e_1,e_2,\dots,e_k$ (in increasing order), we have
\begin{equation*} \begin{aligned}
&\PP\left(\wb_t = \wb^* | \wb_{t-1}\right) - \PP\left(\wb_1 = \wb' | \wb_{t-1}\right) \\
=&\sum_{\left(i_1,\dots,i_k\right) \in \cC^{\left(k\right)}} \left(
p_{\tilde{\wb}_{t}}[i_1] \cdot \frac{p_{\tilde{\wb}_{t}}[i_2]}{1 - p_{\tilde{\wb}_{t}}[i_1]} \cdots \frac{p_{\tilde{\wb}_{t}}[i_k]}{1 -\sum_{j<k} p_{\tilde{\wb}_{t}}[i_j]} - 
p_{\tilde{\wb}_{t}}[e_{i_1}] \cdot \frac{p_{\tilde{\wb}_{t}}[e_{i_2}]}{1 - p_{\tilde{\wb}_{t}}[e_{i_1}]} \cdots \frac{p_{\tilde{\wb}_{t}}[e_{i_k}]}{1 -\sum_{j<k} p_{\tilde{\wb}_{t}}[e_{i_j}]}
\right) \\
&> 0,
\end{aligned} \end{equation*}
the last inequality holds because $p_{\tilde{\wb}_{t}}[i] \geq p_{\tilde{\wb}_{t}}[e_i]$ for all $i < k$ and  $p_{\tilde{\wb}_{t}}[k] > p_{\tilde{\wb}_{t}}[e_k]$, thus for $t \geq 1$:
\begin{equation*} \begin{aligned}
\PP\left(\wb_t = \wb^* | \wb_{t-1}\right) > \PP\left(\wb_t = \wb' | \wb_{t-1}\right) \, \forall \, \wb' \in \cW_{/\wb^*}, \wb_{t-1} \in \cW,
\end{aligned} \end{equation*}
Since greedy decoding selects the $\wb$ with highest probability, $\wb_t^{\mathtt{greedy}} = \wb^*$ for all $t \geq 1$. Additionally,  
\begin{equation*} \begin{aligned}
\PP\left(\wb_t = \wb^* | \wb_{0}\right) &= \sum_{\wb \in \cW} \PP\left(\wb_t = \wb^* |\wb_{t-1} = \wb\right) \PP\left(\wb_{t-1} = \wb |\wb_{0}\right) \\
&> \sum_{\wb \in \cW} \PP\left(\wb_t = \wb' |\wb_{t-1} = \wb\right) \PP\left(\wb_{t-1} = \wb |\wb_{0}\right) \\
&= \PP\left(\wb_t = \wb' | \wb_{0}\right).
\end{aligned} \end{equation*}
This implies $\PP\left(\wb_t = \wb^* | \wb_{0}\right) > \PP\left(\wb_t = \wb' | \wb_{0}\right)$ for all $\wb \in \cW_{/\wb^*}$, and according to  \autoref{th:sc_maj_vote}, majority vote will choose $\wb_{t, N}^{\mathtt{mv}} = \wb^*$ with sufficient large sampling number $N$.
\end{proof}

\begin{lemma}[Convergence Rate for Majority Vote ]\label{lem:maj_convergence}
Given $\tilde{\wb}_t = \wb_{t-1} - \frac{1}{n}\left(\Xb \Xb^\top \wb_{t-1} - \Xb \Yb^\top\right)$, where $\Yb = \wb^* \Xb + \mathbf{\epsilon}$, suppose $\cE_1$ holds and $\frac{2k + \sigma_\epsilon}{n^{1/4}} < \frac13$ , then
\begin{equation*}
\begin{aligned}
\PP\left(\wb_t = \wb^* | \wb_{0}\right) - \max\limits_{\wb' \in \cW_{/\wb^*}}\PP\left(\wb_t = \wb' | \wb_{0}\right)  \geq \frac{p_{\mathtt{trans}}}{p_{\mathtt{trans}} + 1 - p_{\mathtt{recurr}}} \left(1-  \left(p_{\mathtt{recurr}} - p_{\mathtt{trans}}\right)^{t-1} \right).
\end{aligned}
\end{equation*}
Where 
\begin{equation*}\begin{aligned}
p_{\mathtt{trans}} &= \left(1 - \frac{2k + \sigma_\epsilon}{n^{1/4} -\left(2k + \sigma_\epsilon\right)}\right) \frac{1}{d^k}, \\
p_{\mathtt{recurr}} &= \left(1 - \frac{\sigma_\epsilon}{n^{1/4} - \sigma_\epsilon}\right) \left(\frac{n^{1/4} - \sigma_\epsilon}{n^{1/4} - \sigma_\epsilon + d \sigma_\epsilon}\right)^{k}.
\end{aligned}\end{equation*}
\end{lemma}
\begin{proof}
First, when $\wb_{t - 1} = \wb^*$, we have
\begin{equation*} \begin{aligned}
 \begin{cases}
\tilde{\wb}_{t}[i] \geq 1 - \frac{\sigma_\epsilon}{n^{1/4}} & i \leq k \\
\tilde{\wb}_{t}[i] \leq \frac{\sigma_\epsilon}{n^{1/4}} & k < i \leq d 
\end{cases}
\end{aligned} \end{equation*}
Let $\tau =  \frac{\sigma_\epsilon}{n^{1/4}}$. For $p_{\tilde{\wb}_{t}}[i] = \frac{\max\left(0,\tilde{\wb}_{t}\right)}{\sum_{j = 1}^d \max\left(0,\tilde{\wb}_{t}\right)}$ and $i \leq k$:

\begin{equation*} \begin{aligned}
p_{\tilde{\wb}_{t}}[i] \geq \frac{1-\tau}{k\left(1-\tau\right) + d \tau} = \frac{1}{k} \frac{k\left(1-\tau\right)}{k\left(1-\tau\right) + d \tau}
\end{aligned} \end{equation*}
Hence,
\begin{equation*} \begin{aligned}
\PP\left(\wb_t = \wb^* | \wb_{t-1} = \wb^*\right) &= \sum_{\left(i_1,\dots,i_k\right) \in \cC^{\left(k\right)}} \left(
p_{\tilde{\wb}_{t}}[i_1] \cdot \frac{p_{\tilde{\wb}_{t}}[i_2]}{1 - p_{\tilde{\wb}_{t}}[i_1]} \cdots \frac{p_{\tilde{\wb}_{t}}[i_k]}{1 -\sum_{j<k} p_{\tilde{\wb}_{t}}[i_j]}\right) \\
&\geq \frac{\left(\frac{1}{k} - \frac{d \tau}{\left(k\left(1-\tau\right) + d \tau\right)k}\right)^k k!}{\prod_{m = 1}^{k-1}\left(1 - m\left(\frac{1}{k} - \frac{d \tau}{\left(k\left(1-\tau\right) + d \tau\right)k}\right)\right) } \\
&\geq \left(\frac{1 - \tau}{1 + \left(d - 1\right) \tau}\right)^{k}
\end{aligned} \end{equation*}
the last inequality is by let $v = \frac{k\left(1-\tau\right)}{k\left(1-\tau\right) + d \tau}$
\begin{equation*} \begin{aligned}
\frac{\left(\frac{v}{k}\right)^k k!}{\prod_{m = 1}^{k-1}\left(1 - m\frac{v}{k}\right) } 
&\geq \frac{v^k\left(\frac{1}{k}\right)^k k!}{\prod_{m = 1}^{k-1}\left(\left(k - \left(k-1\right)v\right)\left(1 - m\frac{1}{k}\right)\right)} \\
&= \frac{v^k}{\left(k - \left(k-1\right)v\right)^{k - 1}} \frac{\left(\frac{1}{k}\right)^k k!}{\prod_{m = 1}^{k-1}\left(1 - m\frac{1}{k}\right) } \\
&\geq \left(\frac{v}{k - \left(k-1\right)v}\right)^k =  \left(\frac{1 - \tau}{1 - \tau + d \tau}\right)^{k}
\end{aligned} \end{equation*}

Next, for $\wb' \in \cW_{/\wb^*}$ where the index of nonzero elements are $e_1,e_2,\dots,e_k$ (increasing order), we have:
\begin{equation*} \begin{aligned}
&\PP\left(\wb_t = \wb^* | \wb_{t-1}\right) - \PP\left(\wb_1 = \wb' | \wb_{t-1}\right) \\
=&\sum_{\left(i_1,\dots,i_k\right) \in \cC^{\left(k\right)}} \left(
p_{\tilde{\wb}_{t}}[i_1] \cdot \frac{p_{\tilde{\wb}_{t}}[i_2]}{1 - p_{\tilde{\wb}_{t}}[i_1]} \cdots \frac{p_{\tilde{\wb}_{t}}[i_k]}{1 -\sum_{j<k} p_{\tilde{\wb}_{t}}[i_j]} - 
p_{\tilde{\wb}_{t}}[e_{i_1}] \cdot \frac{p_{\tilde{\wb}_{t}}[e_{i_2}]}{1 - p_{\tilde{\wb}_{t}}[e_{i_1}]} \cdots \frac{p_{\tilde{\wb}_{t}}[e_{i_k}]}{1 -\sum_{j<k} p_{\tilde{\wb}_{t}}[e_{i_j}]}
\right) \\
>&\left(\prod_{i = 1}^k p_{\tilde{\wb}_{t}}[i] - \prod_{i = 1}^k p_{\tilde{\wb}_{t}}[e_{i}]\right) \sum_{\left(i_1,\dots,i_k\right) \in \cC^{\left(k\right)}} \left(
\frac{1}{1 - p_{\tilde{\wb}_{t}}[i_1]} \cdots \frac{1}{1 -\sum_{j<k} p_{\tilde{\wb}_{t}}[i_j]}
\right) \\
>& \left(1 - \frac{p_{\tilde{\wb}_{t}}[e_i]}{p_{\tilde{\wb}_{t}}[i]}\right) \prod_{i = 1}^k p_{\tilde{\wb}_{t}}[i] \sum_{\left(i_1,\dots,i_k\right) \in \cC^{\left(k\right)}} \left(
\frac{1}{1 - p_{\tilde{\wb}_{t}}[i_1]} \cdots \frac{1}{1 -\sum_{j<k} p_{\tilde{\wb}_{t}}[i_j]}
\right) \\
=& \left(1 - \frac{p_{\tilde{\wb}_{t}}[e_i]}{p_{\tilde{\wb}_{t}}[i]}\right) \PP\left(\wb_t = \wb^* | \wb_{t-1}\right)
\end{aligned} \end{equation*}
Given that $\PP\left(\wb_t = \wb^* | \wb_{t-1}\right) > \PP\left(\wb_t = \wb' | \wb_{t-1}\right)$ for $\wb' \in \cW_{/\wb^*}$, we have $\PP\left(\wb_t = \wb^* | \wb_{t-1}\right) > \frac{1}{|\cW|} \geq \frac{1}{d^k} $, when $\cE_1$ holds:
\begin{equation*} \begin{aligned}
&\PP\left(\wb_t = \wb^* | \wb_{t-1}\right) - \PP\left(\wb_1 = \wb' | \wb_{t-1}\right) > \left(1 - \frac{\frac{2k + \sigma_\epsilon}{n^{1/4}}}{1 - \frac{2k + \sigma_\epsilon}{n^{1/4}}}\right) \frac{1}{d^k}
\end{aligned} \end{equation*}
Specifically, 
\begin{equation*} \begin{aligned}
&\PP\left(\wb_t = \wb^* | \wb^*\right) - \PP\left(\wb_1 = \wb' | \wb^*\right) > \left(1 - \frac{\frac{\sigma_\epsilon}{n^{1/4}}}{1 - \frac{\sigma_\epsilon}{n^{1/4}}}\right) \PP\left(\wb_t = \wb^* | \wb^*\right)
\end{aligned} \end{equation*}
Therefore,
\allowdisplaybreaks
\begin{equation*} \begin{aligned}
&\PP\left(\wb_t = \wb^* | \wb_{0}\right) - \PP\left(\wb_t =\wb' | \wb_{0}\right) \\
=& \sum_{\wb \in \cW} \left(\PP\left(\wb_t = \wb^* |\wb_{t-1} = \wb\right) - \PP\left(\wb_t = \wb' |\wb_{t-1} = \wb\right) \right) \PP\left(\wb_{t-1} = \wb |\wb_{0}\right) \\
>& \sum_{\wb \in \cW_{/\wb^*}}\left(1 - \frac{2k + \sigma_\epsilon}{n^{1/4} -\left(2k + \sigma_\epsilon\right)}\right)\PP\left(\wb_t = \wb^* |\wb_{t-1} = \wb\right)\PP\left(\wb_{t-1} = \wb |\wb_{0}\right) \\
&+ \left(1 - \frac{\sigma_\epsilon}{n^{1/4} - \sigma_\epsilon}\right) \PP\left(\wb_t = \wb^* | \wb^*\right)\PP\left(\wb_{t-1} = \wb^* |\wb_{0}\right) \\
>& \left(1 - \frac{2k + \sigma_\epsilon}{n^{1/4} -\left(2k + \sigma_\epsilon\right)}\right) \frac{1}{d^k}\sum_{\wb \in \cW_{/\wb^*}}\PP\left(\wb_{t-1} = \wb |\wb_{0}\right) + \left(1 - \frac{\sigma_\epsilon}{n^{1/4} - \sigma_\epsilon}\right)\left(\frac{1 - \tau}{1 - \tau + d \tau}\right)^{k}\PP\left(\wb_{t-1} = \wb^* |\wb_{0}\right) \\
=& \underbrace{\left(1 - \frac{2k + \sigma_\epsilon}{n^{1/4} -\left(2k + \sigma_\epsilon\right)}\right) \frac{1}{d^k}}_{p_{\mathtt{trans}}} \left(1 - \PP\left(\wb_{t-1} = \wb^* |\wb_{0}\right)\right) +  \underbrace{\left(1 - \frac{\sigma_\epsilon}{n^{1/4} - \sigma_\epsilon}\right) \left(\frac{n^{1/4} - \sigma_\epsilon}{n^{1/4} - \sigma_\epsilon + d \sigma_\epsilon}\right)^{k}}_{p_{\mathtt{recurr}}}\PP\left(\wb_{t-1} = \wb^* |\wb_{0}\right) \\
>& \left(p_{\mathtt{recurr}} - p_{\mathtt{trans}}\right)^{t-1} \left(\PP\left(\wb_1 = \wb^* |\wb_{0}\right) - \frac{p_{\mathtt{trans}}}{p_{\mathtt{trans}} + 1 - p_{\mathtt{recurr}}}\right) + \frac{p_{\mathtt{trans}}}{p_{\mathtt{trans}} + 1 - p_{\mathtt{recurr}}} \\
>& \frac{p_{\mathtt{trans}}}{p_{\mathtt{trans}} + 1 - p_{\mathtt{recurr}}} \left(1-  \left(p_{\mathtt{recurr}} - p_{\mathtt{trans}}\right)^{t-1} \right) 
\end{aligned} \end{equation*}
\end{proof}
\subsection{Proof of Theorem \ref{th:sc_limited_n}}\label{subsec:_proof_th_sc_limited_n}

To prove \autoref{th:sc_limited_n}, we first demonstrate that the majority vote algorithm can achieve perfect accuracy with a high probability given a sufficient large sampling number $N$ (by combining \autoref{lem:n_k_1_2p} and \autoref{lem:n_k_1_maj}). Subsequently, for the greedy decoding algorithm, we prove that with high probability, $\wb^{\mathtt{greedy}}_t$ will transition between states $\wb'$ and $\wb''$, where $\wb', \wb'' \neq \wb^*$.

In the following, as we consider the case where $k = 1$, we define $\mathbbm{1}_i = [0,\dots,\underset{\underset{i\text{-th}}{\downarrow}}{1},0,\dots]$ be a vector with a value of $1$ at the $i$-th element and $0$ elsewhere. Without loss of generality, we assume $\wb^* = \mathbbm{1}_1$.

\begin{lemma}\label{lem:n_k_1_2p}
Consider the case where $n = k = 1, \sigma_\epsilon = 0$, and denote the in-context example as $\left(\xb, \wb^\top \xb \right)$. Then:
\begin{equation*} \begin{aligned}
\PP\left(\wb_{t+2} = \wb^* | \wb_{t} = \wb\right) > 0
\end{aligned} \end{equation*}
Holds for all $\wb \in \cW$ with probability at least $1 - \frac{1}{2^{d-1}}$.
\end{lemma}
\begin{proof}
\begin{equation*} \begin{aligned}
\PP\left(\wb_{t+2} = \wb^* | \wb_{t} = \wb\right) = \sum_{\wb' \in \cW} \PP\left(\wb_{t+2} = \wb^* | \wb_{t-1} = \wb'\right)\PP\left(\wb_{t+1} = \wb' | \wb_{t} = \wb\right)
\end{aligned} \end{equation*}

It suffices to demonstrate the existence of a $\wb' \in \cW$, such that $\PP\left(\wb_{t+2} = \wb^* | \wb_{t-1} = \wb'\right)\PP\left(\wb_{t+1} = \wb' | \wb_{t} = \wb\right) > 0$.

Without losing generality, we let $x_1 > 0$, $\wb_t = \mathbbm{1}_l$ and for $\xb = [x_1,x_2,\dots,x_d]$ we let $x_1 > 0$, $x_2 \geq x_3 \dots \geq x_d$. We have:

\begin{equation*} \begin{aligned}
&\tilde{\wb}_{t+1} [i] = \wb_{t} [i] - \sum_{j \in [d]}\left(x_{i}x_{j}\left(\wb_{t-1} [j] - \wb^* [j]\right)\right) \\
&\begin{cases}
    \tilde{\wb}_{t+1} [i] = x_i \left(x_1 - x_l\right) & \text{if } i \neq l \\
    \tilde{\wb}_{t+1} [i] = 1 + x_l \left(x_1 - x_l\right) & \text{if } i = l \\
\end{cases}.
\end{aligned} \end{equation*}

If $x_1 - x_l > 0$, then $\tilde{\wb}_{t+1} [1] > 0$, implying the existence of $\wb' = \wb^*$, such that:
\begin{equation*} \begin{aligned}
&\PP\left(\wb_{t+2} = \wb^* | \wb_{t-1} = \wb'\right)\PP\left(\wb_{t+1} = \wb' | \wb_{t} = \wb\right) \\
=&\PP\left(\wb_{t+2} = \wb^* | \wb_{t-1} = \wb*\right)\PP\left(\wb_{t+1} = \wb* | \wb_{t} = \wb\right) \\
=&\frac{x_{1} \left(x_1 - x_l\right)}{\sum_{i \in [d]} \max{\left(0,\tilde{\wb}_{t+1} [i]\right)}} > 0
\end{aligned} \end{equation*}
If $x_1 - x_l < 0$, we consider the case where $x_d < 0$, which occurs with a probability of at least $1 - \frac{1}{2^{d-1}}$. In this case, we ensure $x_d < 0$ to satisfy $x_d\left(x_1 - x_l\right) > 0$. Subsequently, leveraging the condition $x_1 - x_d > 0$, we can choose $\wb' = \mathbbm{1}_d$ such that:
\begin{equation*} \begin{aligned}
&\PP\left(\wb_{t+2} = \wb^* | \wb_{t-1} = \wb'\right)\PP\left(\wb_{t+1} = \wb' | \wb_{t} = \wb\right) \\
\geq& \frac{x_{d} \left(x_1 - x_l\right)}{\sum_{i \in [d]} \max{\left(0,\tilde{\wb}_{t+1} [i]\right)}} \cdot \frac{x_{1} \left(x_1 - x_{d}\right)}{\sum_{i \in [d]} \max{\left(0,\tilde{\wb}_{t+2} [i]\right)}} > 0
\end{aligned} \end{equation*}

\end{proof}

\begin{lemma}\label{lem:n_k_1_maj}
Consider the case where $n = k = 1, \sigma_\epsilon = 0$, and denote the in-context example as $\left(\xb, \wb^\top \xb \right)$.There exists a $\zeta > 0$ such that for reasoning steps $T > \frac{2\ln{1/2}}{\ln{1 - \zeta}}$ and  sufficient large sampling number $N$, it holds that 
\begin{equation*} \begin{aligned}
\wb_{T, N}^{\mathtt{mv}} = \wb^*,
\end{aligned} \end{equation*}
with probability at least $1 - \frac{1}{2^{d-1}}$.
\end{lemma}

\begin{proof}
Referring to \autoref{lem:n_k_1_2p}, with probability at least $1 - \frac{1}{2^{d-1}}$, $\PP\left(\wb_{t+2} = \wb^* | \wb_{t} = \wb\right) > 0$ holds for all $\wb \in \cW$, define 

\begin{equation*} \begin{aligned}
\zeta &= \min_{\wb \in \cW} \PP\left(\wb_{t+2} = \wb^* | \wb_{t} = \wb\right).
\end{aligned} \end{equation*}
Assume $t = 2q + 1$ (if not, since $\PP\left(\wb_t = \wb^* | \wb_{0}\right) \geq \PP\left(\wb_{t-1} = \wb^* | \wb_{0}\right)$, we can set $t-1 = 2q + 1$) 
\begin{equation*} \begin{aligned}
&\PP\left(\wb_{2q + 1} = \wb^* | \wb_{0}\right) \\
=& \sum_{\wb \in \cW} \PP\left(\wb_{2q + 1} = \wb^* |\wb_{2q - 1} = \wb\right) \PP\left(\wb_{2q - 1} = \wb |\wb_{0}\right) \\
=& \sum_{\wb \in \cW_{/\wb^*}} \PP\left(\wb_{2q + 1} = \wb^* |\wb_{2q - 1} = \wb\right) \PP\left(\wb_{2q - 1} = \wb |\wb_{0}\right) + \PP\left(\wb_{2q + 1} = \wb^* |\wb_{2q -1} = \wb^*\right) \PP\left(\wb_{2q - 1} = \wb^* |\wb_{0}\right) \\
\geq& \zeta \left(1 - \PP\left(\wb_{2q - 1} = \wb^* |\wb_{0}\right)\right) + \PP\left(\wb_{2q - 1} = \wb^* |\wb_{0}\right)  \\
\geq& \left(1 - \zeta\right)^k \left(\PP\left(\wb_{1} = \wb^* |\wb_{0}\right) - 1\right) + 1\geq 1 - \left(1 - \zeta\right)^k
\end{aligned} \end{equation*}
If $k> \frac{\ln{1/2}}{\ln\left(1 - \zeta\right)}$, then $\PP\left(\wb_{2q + 1} = \wb^* | \wb_{0}\right) > 1/2$, and therefore: 
\begin{equation*} \begin{aligned}
\PP\left(\wb_{t} = \wb^* | \wb_{0}\right) > \frac12 >  1 - \PP\left(\wb_{t} = \wb^* | \wb_{0}\right) > \PP\left(\wb_{t} = \wb' | \wb_{0}\right) \, \forall \wb' \in \cW_{/\wb^*}
\end{aligned} \end{equation*}
In this case, by \autoref{th:sc_maj_vote}, with sufficient large sample number $N$, $\wb_{T, N}^{\mathtt{mv}} = \wb^*$.
\end{proof}

\begin{lemma}
Consider the case where $n = k = 1, \sigma_\epsilon = 0$, and denote the in-context example as $\left(\xb, \wb^\top \xb \right)$. Then 
\begin{equation*} \begin{aligned}
\wb_{t}^{\mathtt{greedy}} \neq \wb^* 
\end{aligned} \end{equation*}
holds with probability at least $1 - \frac{2}{d} - \frac{1}{2^{d-1}}$.
\end{lemma}

\begin{proof}
Here, we directly construct a case where, with a high probability, the greedy decoding will become stuck between two stages and fail to reach the state $\wb^*$.

Without loss of generality, we assume $x_1 > 0$, and we select $x_2$ and $x_3$ such that $x_2 = \max_{i > 1}x_i$ and $x_3 = \max_{i > 1} \left(-x_i\right)$. With a probability of $1 - \sum_{r = 1}^{d-1} \frac{1}{r+1} \frac{{d-1 \choose r}}{2^{d-1}} - \frac{1}{2^{d-1}} > 1 - \frac{2}{d} - \frac{1}{2^{d-1}}$, it holds that $x_2 > x_1 > 0$ and $x_3 < 0$.

In this case, 
\begin{equation*} \begin{aligned}
\tilde{\wb}_1[2] &= x_1 x_2 > x_1 x_j = \tilde{\wb}_1[j],
\end{aligned} \end{equation*}
holds for all $j \in [d], j\neq 2$. Then $\wb_{1}^{\mathtt{greedy}} = \wb' \neq \wb^*$ where $\wb' = \mathbbm{1}_2$. Similarly, 
\begin{equation*} \begin{aligned}
\begin{cases}
    \tilde{\wb}_{2} [i] = x_i \left(x_1 - x_2\right) & \text{if } i \neq 2 \\
    \tilde{\wb}_{2} [i] = 1 + x_i \left(x_1 - x_2\right) & \text{if } i = 2 \\
\end{cases}
\end{aligned} \end{equation*}

If $\argmax_{i \in [d]} \tilde{\wb}_{2} [i] = 2$, then $\wb_{2}^{\mathtt{greedy}} = \wb'$, thus for $\wb_{t}^{\mathtt{greedy}} = \wb' \neq \wb^*$ holds when $t \geq 1$.
.
If $\argmax_{i \in [d]} \tilde{\wb}_{2} [i] \neq 2$, as $x_1 - x_2 < 0$,
\begin{equation*} \begin{aligned}
\tilde{\wb}_2[3] = x_3 \left(x_1 - x_2\right) >  x_i \left(x_1 - x_2\right) = \tilde{\wb}_2[j],
\end{aligned} \end{equation*}
holds for all $j \in [d], j\neq 3$. In this case, we have $\wb_2 = \wb'' \neq \wb^*$ where $\wb'' = \mathbbm{1}_3$ and for $\tilde{\wb}_{3}$:
\begin{equation*} \begin{aligned}
\begin{cases}
    \tilde{\wb}_{3} [i] = x_i \left(x_1 - x_3\right) & \text{if } i \neq 3 \\
    \tilde{\wb}_{3} [i] = 1 + x_i \left(x_1 - x_3\right) & \text{if } i = 3 \\
\end{cases}
\end{aligned} \end{equation*}
Similarly, if $\argmax_{i \in [d]} \tilde{\wb}_{3} [i] = 3$, then $\wb_{3}^{\mathtt{greedy}} = \wb''$, thus for $\wb_{t}^{\mathtt{greedy}} = \wb'' \neq \wb^*$ holds when $t \geq 2$. 

If $\argmax_{i \in [d]} \tilde{\wb}_{2} [i] \neq 2$, as $\left(x_1 - x_3\right) > 0$, we know that $\wb_{3}^{\mathtt{greedy}} = \wb'$, then $\wb_{4}^{\mathtt{greedy}} = \wb''$, $\wb_{5}^{\mathtt{greedy}} = \wb'$...

In conclusion, $\wb^{\mathtt{greedy}}_t$ will be either $\wb'$ or $\wb''$ for $t > 0$, thus $\wb^{\mathtt{greedy}}_t \neq \wb^*$ for $t > 0$.
\end{proof}

%% file: icml2025/appendix/prompt.tex
\section{Prompt Examples}\label{appdix:prompt}

\begin{tcolorbox}[title={Prompt For GSM8K with Assigned Token Budget}]
You are a math problem solver. I will give you a problem from the Grade School Math 8K dataset (GSM8K). At the end, provide the final answer as a single integer.

Example: 
    Problem: There are 15 trees in the grove. Grove workers will plant trees in the grove today. After they are done, there will be 21 trees. How many trees did the grove workers plant today?
    Answer (You should choose different reasoning method based on different tokens limit): 
    
        Case 1 (low token budgets, for example 20): We have token limits 20.  The answer is \#\#6\#\#. [END]
        
        Case 2 (medium token budgets, for example 100): We have token limits 100. 21 - 15 = 6. The answer is \#\#6\#\#. [END]
        
        Case 3 (high token budgets, for example 200): We have token limits 200. There are 15 trees originally. Then there were 21 trees after some more were planted. So there must have been 21 - 15 = 6. The answer is \#\#6\#\#. [END]
        
        Case 4 (sufficient token budgets, for example 500): We have token limits 500. There are 15 trees originally. Then there were 21 trees after some more were planted. So there must have been 21 - 15 = 6. [...(more thoughts such as check answer to satisfy tokens limit)]  The answer is \#\#6\#\#. [END]
        
Important: You should try your best to use around \textcolor{red}{\{token\_limit\}} tokens in your reasoning steps.

If you feel like you are finished early, spend the extra tokens trying to double check your work until you are absolutely sure that you have the correct answer.

Here's the problem:

\textcolor{red}{\{problem\}}

Solve this problem, use around  \textcolor{red}{\{token\_limit\}} tokens in your reasoning, provide the final answer as a single integer, and put your final answer in this format: ``The answer is \#\#your answer\#\#.", and end this chat with `[END]'
\end{tcolorbox}

For the MATH dataset, we simply replaced the ``Grade School Math 8K dataset (GSM8K)" (first line in above prompt) with ``MATH."